\documentclass[sigconf]{aamas} 

\usepackage{balance} 

\usepackage{amsmath,amsfonts,bm}

\def\eqref#1{equation~\ref{#1}}

\def\1{\bm{1}}

\DeclareMathAlphabet{\mathsfit}{\encodingdefault}{\sfdefault}{m}{sl}
\SetMathAlphabet{\mathsfit}{bold}{\encodingdefault}{\sfdefault}{bx}{n}

\usepackage{amsthm} 
\newtheorem{theorem}{Theorem}

\newtheorem{lemma}{Lemma}
\newtheorem*{lemma*}{Lemma}
\newtheorem*{claim*}{Claim}

\theoremstyle{definition}
\newtheorem{definition}{Definition}

\newtheorem{assumption}{Assumption}
\theoremstyle{remark}

\usepackage{mathtools, nccmath}
\usepackage{multicol,amsmath,bm}
\usepackage{booktabs} 
\usepackage{subcaption}
\usepackage[linesnumbered,ruled]{algorithm2e}
\usepackage{wrapfig}

\usepackage{tikz}
\usepackage{anyfontsize}
\usetikzlibrary{arrows,matrix,positioning,decorations.pathreplacing,calc}
\usepackage{multirow}
\usepackage{hhline}
\usepackage{comment}
\usepackage{color}
\usepackage{float}
\usepackage{algpseudocode}

\usepackage{lipsum}
\usepackage{mwe}

\usepackage{caption}
\usepackage{subcaption}

\usepackage{hyperref}
\usepackage{url}

\setcopyright{ifaamas}
\acmConference[AAMAS '24]{Proc.\@ of the 23rd International Conference
on Autonomous Agents and Multiagent Systems (AAMAS 2024)}{May 6 -- 10, 2024}
{Auckland, New Zealand}{N.~Alechina, V.~Dignum, M.~Dastani, J.S.~Sichman (eds.)}
\copyrightyear{2024}
\acmYear{2024}
\acmDOI{}
\acmPrice{}
\acmISBN{}

\acmSubmissionID{<<EasyChair submission id>>}

\title{Sample and Communication Efficient Fully Decentralized MARL Policy Evaluation via a New Approach: Local TD update}

\author{Hairi}
\affiliation{
  \institution{University of Wisconsin-Whitewater}
  \city{Whitewater}
  \country{USA}}
\email{hairif@uww.edu}

\author{Zifan Zhang}
\affiliation{
  \institution{North Carolina State University}
  \city{Raleigh}
  \country{USA}}
\email{zzhang66@ncsu.edu}

\author{Jia Liu}
\affiliation{
  \institution{The Ohio State University}
  \city{Columbus}
  \country{USA}}
\email{liu@ece.osu.edu}

\begin{abstract}
In actor-critic framework for fully decentralized multi-agent reinforcement learning (MARL), one of the key components is the MARL policy evaluation (PE) problem, where a set of $N$ agents work cooperatively to evaluate the value function of the global states for a given policy through communicating with their neighbors.
In MARL-PE, a critical challenge is how to lower the sample and communication complexities, which are defined as the number of training samples and communication rounds needed to converge to some $\epsilon$-stationary point.
To lower communication complexity in MARL-PE, a ``natural'' idea is to perform multiple local TD-update steps between each consecutive rounds of communication to reduce the communication frequency. However, the validity of the local TD-update approach remains unclear due to the potential ``agent-drift'' phenomenon resulting from heterogeneous rewards across agents in general.
This leads to an interesting open question: {\em Can the local TD-update approach entail low sample and communication complexities?} In this paper, we make the first attempt to answer this fundamental question.
We focus on the setting of MARL-PE with average reward, which is motivated by many multi-agent network optimization problems.
Our theoretical and experimental results confirm that allowing multiple local TD-update steps is indeed an effective approach in lowering the sample and communication complexities of MARL-PE compared to consensus-based MARL-PE algorithms.
Specifically, the local TD-update steps between two consecutive communication rounds can be as large as 
$\mathcal{O}(1/\epsilon^{1/2}\log{(1/\epsilon)})$ in order to converge to an $\epsilon$-stationary point of MARL-PE. 
Moreover, we show theoretically that in order to reach the optimal sample complexity, the communication complexity of local TD-update approach is $\mathcal{O}(1/\epsilon^{1/2}\log{(1/\epsilon)})$.
\end{abstract}



\newcommand{\BibTeX}{\rm B\kern-.05em{\sc i\kern-.025em b}\kern-.08em\TeX}

\begin{document}


\pagestyle{fancy}
\fancyhead{}


\maketitle 


\section{Introduction} \label{sec:intro}
1) {\bf Background and Motivation:} With the recent success of reinforcement learning (RL) techniques in the dynamic decision-making process \citep{SutBar_18}, 
MARL, a natural extension of RL to multi-agent systems, has also received increasing attention.
Compared to traditional RL, the richness of multi-agent systems has given rise to far more diverse problem settings in MARL, including cooperative, competitive, and mixed MARL (see \citep{ZhaYanBas_21} for an excellent survey). 
In this paper, we are interested in {\em fully decentralized cooperative} MARL, which has found a wide range of applications in the field of networked large-scale systems, such as power networks \citep{CheQuTan_22,RieMooSch_00}, autonomous driving \citep{YuWanXu_19,ShaShaSha_16}, wireless network \citep{WeiWanLi_22} and so on.
A defining feature of fully decentralized cooperative MARL is that all agents in the system collaborate to learn a joint policy to maximize long-term system-wide total rewards through communicating with each other.
However, due to the decentralized nature (i.e., lack of a centralized infrastructure) of fully decentralized cooperative MARL, the collaboration between the agents can only rely on some speical algorithmic designs to induce a ``consensus'' that can be reached by all agents.

In a consensus-based actor-critic framework, one of the key components is the MARL policy evaluation (PE) problem, where a set of $N$ agents work cooperatively to evaluate the value function of the global states for a given joint policy.
Just as the PE problem in single-agent RL, temporal difference (TD) learning~\citep{Sut_88} has been the prevailing method for MARL-PE thanks to its simplicity and effectiveness.
Simply speaking, the key idea of TD learning is to learn the value function by using the Bellman equation to bootstrap from the current estimated value function.

However, as mentioned earlier, the decentralized nature of the MARL-PE problem necessitates communication among agents for TD learning.
Hence, a critical challenge in consensus-based MARL-PE is how to lower the {\em sample and communication complexities,} which are defined as the required number of training samples and rounds of communications between neighboring agents to converge to an $\epsilon$-stationary point of the MARL-PE problem.

To lower communication complexity for solving MARL-PE problems, a ``natural'' idea is to use an ``infrequent communication'' approach where we perform multiple local TD-update steps between each consecutive rounds of communication to reduce the communication frequency.
However, the validity of the ``local TD-update'' approach remains unclear due to the potential ``agent-drift'' phenomenon resulted from heterogeneous rewards across agents (more on this soon).
This leads to two interesting open questions: 
\begin{itemize}
\item[1)] {\em Can the local TD-update approach achieve low sample and communication complexities for solving MARL-PE?}
\smallskip
\item[2)] {\em If the answer to 1) is ``yes,'' how does the local TD-steps approach perform in comparison to other approaches?}
\end{itemize}
In this paper, we make the first attempt to answer the above open questions.
However, unlike conventional MARL research that adopts discounted reward, in this paper, we are particularly interested in the cooperative MARL setting with {\em average reward}~\cite{HaiLiuLu_22,QuLinWie_20,ZhaYanLiu_18,TsiVan_99,TsiVan_02}.
The average reward setting of MARL-PE is motivated by and highly relevant for many multi-agent and network optimization problems that care about ``average performances'' (e.g., average throughput, average latency, and average energy consumption in multi-hop wireless networks).

\smallskip
2) {\bf Technical Challenges:} 
Answering Questions 1) and 2) above is highly non-trivial due to several technical challenges in the convergence analysis of the local TD-update approach.
Notably, it is easy to see that the structure of TD learning in consensus-based cooperative MARL resembles that of decentralized stochastic gradient descent (DSGD) method in consensus-based decentralized optimization\citep{NedOzd_09,LiaZhaZha_17,PuNed_21}.
Thus, it is tempting to believe that one can borrow convergence analysis techniques of DSGD and apply them in TD learning.
However, despite such similarities, there also exist significant differences between TD learning in MARL and DSGD.

\begin{list}{\labelitemi}{\leftmargin=1em \itemindent=-0.09em \itemsep=.2em}

\item {\em Structural Differences:}
First, we note that TD learning is {\em not} a true gradient-based method since TD error is not a gradient estimator of any static objective function which is well-defined in a consensus-based decentralized optimization problem.
Also, in decentralized optimization, the gradient terms are often assumed to be bounded. 
However, when using approximation for value function in TD learning, TD-errors can not be assumed to be bounded without further assuming that the approximation parameters lie in some compact set.

\item {\em Markovian Noise in TD Learning:}
In RL/MARL problems, there exists an underlying Markovian dynamic process across time steps, where the state distribution may differ at different time steps.
By contrast, in decentralized optimization, it is often safe to assume that the data at each agent are independently distributed.
Thus, it is not possible to directly apply convergence analysis techniques of decentralized optimization in TD learning for MARL-PE.
The coupling and dependence among samples renders the convergence analysis of TD learning in MARL far more challenging.

\item {\em ``Agent-Drift'' Phenomenon:}
Due to heterogeneity nature of the rewards across agents, executing multiple local TD-update steps would inevitably pull the value functions toward the direction of local value functions rather than the global value function, leading to the ``agent-drift'' phenomenon. 
Hence, it is unclear under such "tug of war" whether local TD-update steps help or hurt the convergence of TD learning in MARL-PE. 
Because of the agent-drift effect, the number of local TD update steps has to be chosen judiciously to mitigate the potentially large divergence of the value functions among agents between consecutive communication rounds.
\end{list}

\smallskip
3) {\bf Main Results and Contribution:} 
The main contribution of this paper is that we overcome the above challenges in analyzing the upper bounds of the sample and communication complexities for the local TD-update approach in cooperative fully decentralized MARL-PE.
By doing so, we shed light on the effect of local TD-update steps in the consensus-based TD learning in MARL-PE with average reward. 
We summarize our main results in this paper as follows:

\begin{list}{\labelitemi}{\leftmargin=1em \itemindent=-0.09em \itemsep=.2em}
\item Both theoretically and empirically, we show that allowing multiple local TD-update steps is indeed a valid approach that can significantly lower communication complexities of MARL-PE compared to vanilla consensus-based  decentralized TD learning algorithms \cite{DoaMagRom_21,DoaMagRom_19,ZhaYanLiu_18}. 
Specifically, we show that under the condition of  achieving $\mathcal{O}(1/\epsilon\log^{2} (1/\epsilon))$ sample complexity (which differs from the state-of-the-art sample complexity only by a log factor), the local TD-update approach can allow up to $\mathcal{O}(1/\epsilon^{1/2}\log (1/\epsilon))$ local TD-update steps and the communication complexity upper bound is $\mathcal{O}(1/\epsilon^{1/2}\log (1/\epsilon))$. 
Compared to vanilla algorithms, this improves the communication complexity by a factor of $\mathcal{O}(1/\epsilon^{1/2})$.

\item In comparison with another notable batching approach, we show that the local TD-update approach not only matches the communication complexity of the batching approach, but also achieves a better sample complexity than that of the batching approach~\cite{HaiLiuLu_22} by a factor of $\mathcal{O}(1/\epsilon^{1/2})$ in average reward setting.
Our extensive empirical results also verify the performance of the local TD-update approach and confirm our theoretical results compared to the vanilla TD learning and batching approaches with both synthetic and real-world datasets.
\end{list}

The rest of the paper is organized as follows.
In Section~\ref{sec:related}, we review the literature to put our work in comparative perspectives.
In Section~\ref{sec:formulation}, we present the system model and formulation of the MARL-PE problem in the average reward setting.
In Section~\ref{sec:alg}, we introduce the decentralized TD learning algorithm with multiple local TD-update steps for MARL-PE.
In Section~\ref{sec:convergence}, we provide the theoretical convergence analysis for the decentralized TD learning algorithm with multiple local TD-update steps. In addition, we provide comparisons of both sample and communication complexities of the proposed local TD-update approach with other methods.
Section~\ref{sec:numerical} presents numerical results and Section~\ref{sec:conclusion} concludes this paper.
Due to space limitation, some proof details and additional experiments 
are relegated to the supplementary material .

\section{Related work} \label{sec:related}


In this section, we provide an overview on two lines of research that are related to this work: i) multi-agent reinforcement learning policy evaluation; and ii) single-agent RL policy evaluation.

\smallskip
{\bf 1) Multi-agent reinforcement learning policy evaluation:}
To our knowledge, the work in \citep{ZhaYanLiu_18} proposed the first fully decentralized multi-agent actor-critic algorithm using TD learning in the critic step, which solves the PE problem in average reward setting. However, the convergence results for both its critic and actor steps are asymptotic. 
Finite-time analysis of MARL-PE problem using distributed TD learning algorithm has been first studied in \citep{DoaMagRom_19} under the i.i.d. sampling assumption, and later the work in \citep{DoaMagRom_21} generalized the result to Markovian sampling assumption only in discounted reward settings. 
In \citep{LinZhaYan_19}, a compressed algorithm is proposed where, instead of sending a vector, only a single entry is sent during communication. 
However, their communication complexity (i.e., the number of communication rounds) remains the same as sample complexity and the convergence is only asymptotic. 
In \citep{CheZhaGia_18}, a lazy communication algorithm is proposed assuming a central controller, which is different from the fully decentralized setting that we consider in this paper.

It is worth noting that many of the above existing distributed TD learning algorithms \citep{ZhaYanLiu_18,DoaMagRom_19,DoaMagRom_21} for MARL-PE perform {\em frequent} consensus rounds (i.e., one round of communication per local TD update) to share the value functions among neighbors. 
Specifically, in these algorithms, agents share the value functions to their neighbors in every sampling step, which causes the communication complexity to be the same as the sample complexity. 
In this paper, we consider an infrequent communication framework that allows the agents to do multiple local TD-update steps and communicate with the neighbors once every $K(\gg 1)$ rounds. 
In \citep{HaiLiuLu_22,CheZhoChe_21}, complete actor-critic algorithms have been proposed and the batching approach has been used in the critic step, which corresponds to MARL-PE, in discounted and average reward respectively. 
In this batching approach \cite{HaiLiuLu_22}, consensus is performed in every $M=\mathcal{O}(1/\epsilon)$ samples, which in return only requires $O(1/\epsilon^{1/2}\log(1/\epsilon))$ communication complexity. Detailed discussions on the comparison of the local TD approach and batching approach is provided in Section \ref{subsec: discussion}.

We also remark that there exists another class of approaches \citep{ZhaLiuLiu_21,MacCheZaz_14,LeeYooHov_18,WaiYanWan_18,RenHau_19} that solve the MARL-PE problem by formulating MARL-PE into optimizing projected Bellman error or its variants, where the proposed algorithms require frequent communications. 
This class of algorithms do not use the on-policy TD learning approach as we do in our paper. 
In \citep{KimMooHos_19}, the paper optimizes communication in order to comply with the bandwidth restriction and minimize the collision between pair-wise channels. 
However, this work adopts centralized learning and distributed execution paradigm, where in our paper, the learning process is fully decentralized.

\smallskip
{\bf 2) Single-agent reinforcement learning policy evaluation:}
For single-agent RL, policy evaluation problems have been extensively studied in terms of asymptotic convergence \citep{TsiVan_97,TsiVan_99,TsiVan_02} for both discounted and average reward settings, later finite-time convergence under i.i.d. sampling assumption~\cite{LakSze_18} and under Markovian sampling assumption using different techniques \citep{SriYin_19,BhaRusSin_18} in discounted reward setting. 
Further, using batching TD learning~\citep{XuWanLia_20} yields state-of-the-art sample complexity $O(1/\epsilon\log(1/\epsilon))$ in the discounted reward setting. For average reward setting, \cite{QiuYanYe_21} yields a sample complexity of $\mathcal{O}(1/\epsilon^{2}\log^{3}(1/\epsilon))$, where the sample complexity is worse than that in our multi-agent setting. To the best of our knowledge, the sample complexity of $\mathcal{O}((1/\epsilon)\log^{2}(1/\epsilon))$ in  \cite{SriYin_19} is the state-of-the-art sample complexity for the single agent average-reward RL policy evaluation problem. 
However, there is no notion of ``communication with other agents'' due to the single-agent nature. 
Thus, results in this area, though related, are not directly comparable to our work in terms of communication complexity.
\section{Distributed policy evaluation in multi-agent reinforcement learning} \label{sec:formulation}
Throughout this paper, $\|\cdot\|$ denotes the $\ell_2$-norm for vectors and the $\ell_2$-induced norm for matrices. $\|\cdot\|_{F}$ denotes the Frobenius norm for matrices. $(\cdot)^{T}$ denotes the transpose for a matrix or a vector.  

\subsection{System Model}
Consider a multi-agent system with $N$ agents, denoted by $\mathcal{N}=\{1,\cdots,N\}$, operating in a networked environment. 
Let $\mathcal{E}$ be the edge set for a given network $\mathcal{G}=(\mathcal{N},\mathcal{E})$. 
To formulate our MARL problem and facilitate our subsequent discussions, we first define the notion of networked multi-agent Markov decision process (MDP) in the average reward setting as follows.

\begin{definition}[Networked Multi-Agent MDP]
Let $\mathcal{G}=(\mathcal{N},\mathcal{E})$ be a communication network that connects $N$ agents. A networked multi-agent MDP is defined by following five-tuple: \[(\mathcal{S},\{\mathcal{A}^{i}\}_{i\in\mathcal{N}},P,\{r^{i}\}_{i\in\mathcal{N}},\mathcal{G}),\] where $\mathcal{S}$ is the global state space, $\mathcal{A}^{i}$ is the action set for agent $i$. Let $\mathcal{A}=\prod_{i\in\mathcal{N}}\mathcal{A}^i$ be the joint action set of all agents. 
$P:\mathcal{S}\times\mathcal{A}\times\mathcal{S}\to[0,1]$ is the global state transition function and $r^{i}: \mathcal{S}\times\mathcal{A}$ is the local reward function for agent $i$. 
\label{def: model}
\end{definition}

In this paper, we assume that the global state space $\mathcal{S}$ is finite.
We also assume that at time step $t\ge0$, all agents can observe the current global state $s_t$. 
However, each agent can only observe its own reward $r^{i}_{t+1}$, i.e., agents do not observe or share rewards with other agents. 
Each agent $i\in\mathcal{N}$ receives a deterministic reward $r^{i}(s,a)$  given the global state $s$ and joint action $a$ \footnote{For simplicity of the presentation, we assume that the rewards are deterministic. For more general stochastic rewards, the results are straightforward.}. 

In our MARL system, each agent chooses its action following its local policy $\pi^i$ that is conditioned on the current global state $s$, i.e., $\pi^{i}(a^{i}|s)$ is the probability for agent $i$ to choose an action $a^{i}\in\mathcal{A}^{i}$. 
Then, the joint policy $\pi: \mathcal{S}\times\mathcal{A}\to[0,1]$ can be written as $\pi(a|s)=\prod_{i\in\mathcal{N}}\pi^{i}(a^{i}|s)$.

The global long-term average reward for a given joint policy $\pi$ in average reward setting is defined as follows:
\begin{align}
J_{\pi}&=\lim_{T\to\infty}\frac{1}{T} \mathbb{E}\left( \sum_{t=0}^{T-1}\frac{1}{N}\sum_{i\in\mathcal{N}} r^{i}_{t+1}\right) \nonumber \\
&=\sum_{s\in\mathcal{S}}d(s)\sum_{a\in\mathcal{A}}\pi(a|s)\cdot \bar{r}(s,a), \label{eq: steady_obj}
\end{align}
where $d(\cdot)$ is the steady state distribution, which is guaranteed to exist due to the Assumption \ref{ass: dis} below, and $\bar{r}(s,a)=\frac{1}{N}\sum_{i\in\mathcal{N}}r^{i}(s,a)$. In other words, in the average reward setting, $J_{\pi}$ evaluates the performance of the given policy $\pi$ at steady state as given in (\ref{eq: steady_obj}).

\subsection{Technical Assumptions}
We now state the following assumptions for the MARL system described above. 
\begin{assumption}
For the given policy $\pi$, we assume the induced Markov chain $\{s_t\}_{t\ge0}$ is irreducible and aperiodic.
\label{ass: dis}
\end{assumption}

\begin{assumption}
The reward $r^{i}_t$ is uniformly bounded by a constant $r_{\max}>0$ for any $i\in\mathcal{N}$ and $t\ge0$.
\label{ass: r_bou}
\end{assumption}

\begin{assumption}
Let $A$ be a consensus weight matrix for a given communication network $\mathcal{G}$. 
There exists a positive constant $\eta>0$ such that $A\in \mathbb{R}^{N\times N}$ is doubly stochastic and $A_{ii}\ge\eta$, $\forall i\in\mathcal{N}$. Moreover, $A_{ij}\ge\eta$ if $i,j$ are connected, otherwise $A_{ij}=0$. 
\label{ass: gra}
\end{assumption}

\begin{assumption}
The global value function is parameterized by linear functions, i.e., $V(s;w)=\phi(s)^{\top}w$ where \[\phi(s) = [\phi_1(s),\cdots,\phi_n(s)]^{\top}\in \mathbb{R}^{n}\] is the feature vector associated with the state $s\in\mathcal{S}$. We typically assume the dimension of the vector is smaller than the cardinality of the state space, i.e. $n<|\mathcal{S}|$. The feature vectors $\phi(s)$ are uniformly bounded for any $s\in\mathcal{S}$. Without loss of generality, we assume that $\|\phi(s)\| \le 1$. Furthermore, the feature matrix $\Phi\in \mathbb{R}^{|\mathcal{S}|\times n}$ is full column rank. Also, for any $u\in \mathbb{R}^{n}$, $\Phi u\neq \mathbf{1}$, where $\mathbf{1}$ is an all-one vector. 
\label{ass: fea}
\end{assumption}

Assumption~\ref{ass: dis} guarantees that there exists a unique stationary distribution over $\mathcal{S}$ for the induced Markov chain by the given policy $\pi$. In other words, it guarantees that the steady state distribution $d(\cdot)$ induced by the policy $\pi$ is well defined. 
Assumption~\ref{ass: r_bou} is common in the RL literature (see, e.g., \citep{ZhaYanLiu_18,XuWanLia_20,DoaMagRom_19}) and easy to be satisfied in many practical MDP models with finite state and action spaces.
Assumption~\ref{ass: gra} is standard in the distributed multi-agent optimization literature~\citep{NedOzd_09}. 
This assumption says that non-zero entries of the weight matrix $A$ needs to be lower bounded by a positive value $\eta$. Note that this characterization of the weight matrix is a rich representation, as for the same graph/topology $\mathcal{G}$, the weights can vary, which correspond to different consensus effects.
Assumption~\ref{ass: fea} on features is standard and has been widely adopted in the literature, e.g., \citep{TsiVan_99,ZhaYanLiu_18,QiuYanYe_21,SriYin_19,HaiLiuLu_22}. The goal of this assumption is to approximate the value function as follows:
\begin{align}
V(s)\approx V(s;w)=\phi(s)^{\top} w \nonumber
\end{align}
where $\phi(s)$ is the aforementioned feature vector associated with state $s\in\mathcal{S}$.

\section{Decentralized TD Learning with local TD-Update Steps for MARL-PE} \label{sec:alg}

In this section, we introduce the decentralized TD learning algorithm with local TD-update steps (i.e., infrequent communication), which is illustrated in Algorithm \ref{alg: local_average} \footnote{For simplicity, we present TD(0) in our paper, the algorithm and theoretical results can be generalized to TD($\lambda$) straightforwardly.}.
Given a joint policy $\pi$, the goal of the MARL-PE in the decentralized setting is that the agents collaborate in a consensus manner to characterize the global value function.
Specifically, each agent $i$ maintains a value function approximation parameter $w^{i}$ locally, which estimates the global value function as follows:
\begin{align}
V(s;w^{i})=\phi(s)^{\top}w^{i}. \nonumber
\end{align}

The local TD-update algorithm for MARL-PE contains two loops. The outer loop is the communication rounds, where consensus update (Line~12 in Algorithm \ref{alg: local_average}) is performed for $L$ rounds in total. 
The inner loop is local TD-update steps (Line~10 in Algorithm \ref{alg: local_average}), which are executed $K$ times in between consecutive communication rounds. 
Locally, each agent performs local TD-updates within each communication round $l \in \{0,\cdots, L-1\}$ as follows:
\begin{align}
w^i_{l,k+1}&=w^i_{l,k}+\beta\cdot\delta^{i}_{l,k}\cdot\phi(s_{l,k}), \label{eq: local_td}
\end{align}
where $\beta>0$ is the constant step size and $\delta^{i}_{l,k}$  is the local TD error, which is defined as follows
\begin{align}
\delta^{i}_{l,k}:=r^{i}_{l,k+1}-\mu^{i}_{l,k}+\phi(s_{l,k+1})w^{i}_{l,k}-\phi(s_{l,k})w^{i}_{l,k}, \nonumber 
\end{align}
and $\mu^{i}_{l,k}$ tracks the local average reward, which is updated as follows
\begin{align} \label{eqn:mu}
\mu^{i}_{l,k+1}=(1-\beta)\mu^{i}_{l,k}+\beta r^{i}_{l,k+1}. 
\end{align}
We remark that Eq.~(\ref{eqn:mu}) is the {\em key difference} between the average reward setting and the conventional discounted reward setting in MARL-PE.
In the discounted reward setting, there is no $\mu^{i}$-terms. 
The use of the $\mu^{i}$-term is to keep track of the local average reward for agent $i$.
Surprisingly, we will show later that consensus and finite-time convergence results on $w^{i}$ parameters can be obtained without performing consensus on these $\mu^{i}$ terms.
We also note that each execution of Eq.~(\ref{eq: local_td}) is considered performing one local TD learning step.
Within each inner loop, this local TD update step is performed $K$ times.

Due to the privacy of the reward signals in the fully decentralized setting, the agents are unable to access the rewards of any other agents, let alone the average rewards. 
Therefore, communication/sharing of the value function approximation parameters among the neighbors 
is necessary \citep{ZhaYanLiu_18,DoaMagRom_19,CheZhoChe_21,HaiLiuLu_22}. 
This step is often referred to as {\em consensus} update, which is defined as follows:
\begin{align}
w^{i}_{l+1,0}&=\sum_{j\in\mathcal{N}_i}A_{ij}w^{j}_{l,K},
\label{eq: consensus}
\end{align}
where $\mathcal{N}_i$ denotes the set of neighbors for agent $i$.
In other words, after performing $K$ local TD-update steps, each agent shares its parameter with the neighbors, receives the ones from the neighbors, and then updates its own parameter in a weighted aggregation as shown in Eq.~(\ref{eq: consensus}).

We note that in our algorithm, the infrequent communication is achieved by agents communicating with neighbors periodically with the period being $K$. We also note that when $K=1$, our algorithm reduces to the vanilla distributed TD learning algorithm \citep{DoaMagRom_19,DoaMagRom_21,ZhaYanLiu_18}. 
Therefore, the vanilla distributed TD learning can be viewed as a special case of our proposed algorithm.


\begin{algorithm}[t]
  \SetKwInOut{Input}{Input}
  \SetKwInOut{Output}{Output}
  \SetKwFor{ParFor}{for}{do in parallel}{end for}
  \Input{Initial state $s_0,\pi=\{\pi^{i}|i\in\mathcal{N}\}$, feature map $\phi$, initial parameters $\{w^{i}_{0,0},\mu^{i}_{0,0}| i\in\mathcal{N}\}$, step size $\beta$, communication round number $L$, local step number $K$}
  \BlankLine
  \For{$l=0,\cdots,L-1$}{
  $s_{l,0}=s_{l-1,K}$  (when $l=0$ and $k=0$ , $s_{l,k}=s_0$)\;
  \For{$k=0,\cdots,K-1$}{ 
	\ParFor{ all $i\in\mathcal{N}$}{
	Execute action $a^{i}_{l,k}\sim\pi^{i}(\cdot|s_{l,k})$\;
	Observe the state $s_{l,k+1}$ and reward $r^{i}_{l,k+1}$\;
	Update $\delta^{i}_{l,k}\leftarrow r^{i}_{l,k+1}-\mu^{i}_{l,k}+\phi(s_{l,k+1})^{T}w^{i}_{l,k}-\phi(s_{l,k})^{T}w^{i}_{l,k}$\;
 	Update $\mu^{i}_{l,k+1}\leftarrow \beta r^{i}_{l,k+1}+(1-\beta)\mu^{i}_{l,k}$\;
    Local TD-update Step: $w^{i}_{l,k+1}\leftarrow w^{i}_{l,k}+\beta\delta^{i}_{l,k}\cdot\phi(s_{l,k})$\;
	}
  }
  \ParFor{ all $i\in\mathcal{N}$}{
  Consensus Update: $w^{i}_{l+1,0}\leftarrow \sum_{j\in\mathcal{N}_i} A(i,j)\cdot w^{j}_{l,K}$\;
  }
  }
 
  \Output{$ \{w^{i}_{L,0}|i\in\mathcal{N}\}$}
  \caption{Decentralized TD Learning with periodic local TD-update steps}
  \label{alg: local_average}
\end{algorithm}

\section{Convergence Analysis of The Local TD-Update Approach for MARL-PE} \label{sec:convergence}
In this section, we present the convergence results for Algorithm \ref{alg: local_average}, which further imply both the sample and communication complexities of the local TD-update approach for MARL-PE.
To characterize the convergence, we define the following quantities:

\begin{align}
    &\Psi:=\mathbb{E}[(\phi(s')-\phi(s))\phi^{\top}(s)] \quad \text{and} \quad \nonumber \\
    & b:=\frac{1}{N}\mathbb{E}[\phi(s)(\sum_{i\in\mathcal{N}}r^{i}(s,a)-J_{\pi})], 
    \label{eq: Psi_def}
\end{align}
where $J_{\pi}$ is defined in Eq.~(\ref{eq: steady_obj}).
The expectations in Eq.~(\ref{eq: Psi_def}) are taken over the steady state distribution induced by the given joint policy, which is guaranteed to exist due to Assumption~\ref{ass: dis}, stationary action policy $a\sim \pi(\cdot|s)$ and state transition probability $s'\sim P(\cdot|s,a)$. 
Furthermore, we define 
\begin{align}
w^{*}=-\Psi^{-1}b, \label{eq: ode}
\end{align}
where the invertibility is due to $\Psi$ being negative definite \citep{TsiVan_99,QiuYanYe_21,HaiLiuLu_22}. 
Consequently, $\forall s, \forall k\ge \tau(\beta)$, we define mixing time $\tau(\beta)$ as the time index $k$ that satisfies the following relationship:
\begin{align}
    \|\Psi-\mathbb{E}[(\phi(s_{k+1})-\phi(s_{k}))\phi^{\top}(s)|s_0=s]\|&\le \beta,
    \label{eq: mixing time}
\end{align}
where the expectation is taken over appropriate distributions. We note that under the Assumption~\ref{ass: dis}, by \citep[Theorem 4.9]{LevPer_17}, the Markov chain mixes at a geometric rate, which implies $\tau(\beta)=\mathcal{O}(\log \frac{1}{\beta})$.


\subsection{Supporting Lemmas}
Before presenting our main theorem, we introduce two useful lemmas. 
Our strategy of convergence analysis is to divide the convergence error into two parts. They are the consensus error, which is defined as the agent's parameters deviation from the average parameter, and convergence error of the average parameter to the solution of the ODE in Eq.~(\ref{eq: ode}).

First, we define the average of the parameters to be $\bar{w}_{l,k}=\frac{1}{N}\sum_{i\in\mathcal{N}}w^{i}_{l,k}$ for any communication round $l\in \{0,\cdots,L-1\}$ and local step $ k\in \{0,\cdots, K-1\}$ and similarly $\bar{\mu}_{l,k}=\frac{1}{N}\sum_{i=1}^{N}\mu^{i}_{l,k}$. Then, we define the consensus error for agent $i$ as:
\begin{align}
    Q^{i}_{l,k}:=w^{i}_{l,k}-\bar{w}_{l,k} \label{def: con_err}
\end{align}
and the matrix form is $Q_{l,k}=[Q^{1}_{l,k},\cdots,Q^{N}_{l,k}]\in \mathbb{R}^{n\times N}$.

We provide an upper bound for the consensus error generated by Algorithm \ref{alg: local_average} in the following lemma.
\begin{lemma}
Suppose that Assumptions~\ref{ass: r_bou}--\ref{ass: fea} hold. 
For the consensus error generated by Algorithm~\ref{alg: local_average}, if $\beta K\le\min\{\frac{1}{2},\frac{\eta^{N-1}}{4(1-\eta^{N-1})}\}$, it then holds that
\begin{align}
\|Q_{L,0}\|\le \kappa_1 \rho^{L}\|Q_{0,0}\|+ \frac{\kappa_2\beta K}{1-\rho}, 
\label{eq: con_err}
\end{align}
where $\kappa_1=\frac{2N^{2}(1+\eta^{-(N-1)})}{1-\eta^{N-1}}$, $\kappa_2=8(1+\eta^{-(N-1)})N^{\frac{5}{2}}r_{\max}$ and $\rho:=(1+4\beta K)(1-\eta^{N-1})$. By the condition on $\beta K$, we have $0<\rho<1$.
\label{lem: con_err}
\end{lemma}

The first term in Lemma~\ref{lem: con_err} shows that even if the parameters are not set to be the same initially, the effect of the initial consensus error will vanish exponentially fast as the round of communication $L$ goes to infinity. The second term is linear with respect to $\beta K$, which resembles the constant term in optimization using stochastic gradient descent (SGD) with constant step-sizes. This product term dictates the consensus error and the error level that the algorithm converges to, see discussion on Figure \ref{fig: local_steps} for more details.
Next, we provide a lemma that characterizes the convergence of the average parameter $\bar{w}_{l,k}$ to the TD fixed point defined in Eq.~(\ref{eq: ode}). 
\begin{lemma}
Suppose Assumptions \ref{ass: dis}-\ref{ass: fea} hold. 
For the $w$-parameters generated by Algorithm \ref{alg: local_average}, we have following result for the average of the $w$-parameters: 
\begin{align}
&\mathbb{E}[\|\bar{w}_{L,0}-w^{*}\|^{2}] \nonumber \\
\le & c_2(1-c_1\beta)^{KL-\tau(\beta)}\Big(\sqrt{\|\bar{w}_{0,0}-w^{*}\|^{2}+(\bar{\mu}_{0,0}-J_{\pi})^{2}} \nonumber \\
& +\frac{r_{\max}}{3} \Big)^{2}+c_3\beta \tau(\beta),
\label{eq: w_bar_w_star}
\end{align}
where $c_1, c_2, c_3>0$ are constants that are independent of step-size $\beta$, local TD-update step $K$ and communication round $L$; and $\tau(\beta)=\mathcal{O}(\log \frac{1}{\beta})$ is the mixing time. 
The specified expressions of the constants $c_1$, $c_2$, and $c_3$ can be found in supplementary material.
\label{lem: ave_con}
\end{lemma}
The average parameter $\bar{w}_{L,0}=\frac{1}{N}\sum_{i\in\mathcal{N}}w^{i}_{L,0}$ corresponds to the updates after $K\times L$ samples and $L$ communication rounds. Lemma \ref{lem: ave_con} shows that $\bar{w}_{L,0}$ converges to solution of the ODE with the rate given by the right-hand-side (RHS) of Eq.~(\ref{eq: w_bar_w_star}). 

\subsection{Main Results}
Now, we state the main convergence result of Algorithm~\ref{alg: local_average}:

\begin{theorem}
Suppose that Assumptions~\ref{ass: dis}-\ref{ass: fea} hold. 
For the given policy, consider the output parameters $\{w^{i}_{L,0}|i\in\mathcal{N}\}$ generated by Algorithm \ref{alg: local_average}. 
If $\beta K\le\min\{\frac{1}{2},\frac{\eta^{N-1}}{4(1-\eta^{N-1})}\}$, 
it then follows that:
\begin{align}
&\mathbb{E}\bigg[\sum_{i=1}^{N}\|w^{i}_{L,0}-w^{*}\|^{2} \bigg]\le
2n\bigg(\kappa_1 \rho^{L}\|Q_{0,0}\|+ \frac{\kappa_2\beta K}{1-\rho}\bigg)^{2} \nonumber \\ &+2N\left(c_2(1-c_1\beta)^{KL-\tau(\beta)}(\sqrt{\|\bar{w}_{0,0}-w^{*}\|^{2}+(\bar{\mu}_{0,0}-J_{\pi})^{2}}\right. \nonumber \\
&\left. +\frac{r_{\max}}{3})^{2}+c_3\beta \tau(\beta)\right), 
\label{eq: convergence}
\end{align}
where $\kappa_1,\kappa_2,c_1,c_2,c_3>0, 0<\rho<1 $ are constants, and $\bar{w}_{0,0}=\frac{1}{N}\sum_{i\in\mathcal{N}}w^{i}_{0,0}$, $\bar{\mu}_{0,0}=\frac{1}{N}\sum_{i\in\mathcal{N}}\mu^{i}_{0,0}$ and $Q_{0,0}$ is the initial consensus error defined in Eq.~(\ref{def: con_err}). 
Furthermore, by letting 
\begin{align*}
\beta\!=\!\Theta(\epsilon\log^{-1}(1/\epsilon)), K\!=\!\Theta(1/\epsilon^{1/2}\log(1/\epsilon)),  L\!=\!\Theta(1/\epsilon^{1/2}\log(1/\epsilon)), 
\end{align*}
we have $\mathbb{E}[\sum_{i=1}^{N}\|w^{i}_{L,0}-w^{*}\|^{2}]= \mathcal{O}(\epsilon)$. The sample complexity is $KL=\mathcal{O}(1/\epsilon\log^{2} (1/\epsilon))$ and the communication complexity is $L=\mathcal{O}(1/\epsilon^{1/2}\log(1/\epsilon))$.
\label{thm}
\end{theorem}

Note that due to the use of a double-loop structure in Algorithm \ref{alg: local_average}, the parameter $w^{i}_{L,0}$ of agent $i$ corresponds to the result after $K\times L$ samples. 
We remark that to the best of our knowledge, the state-of-the-art sample complexity for the average reward RL in single agent setting is $\mathcal{O}((1/\epsilon)\log^{2}(1/\epsilon))$ \cite{SriYin_19}. The sample complexity of our algorithm in decentralized multi-agent setting, {\em matches this sample complexity in the single-agent setting}.

\subsection{Discussion}
\label{subsec: discussion}
In this section, we provide a comparison of the proposed local TD-update step approach with vanilla and batching approaches in terms of both sample and communication complexities. 

{\bf 1) Sample complexity in comparison with single agent setting:} The sample complexity of our algorithm matches the state-of-the-art sample complexity in the single-agent setting.
Also, compared to the single-agent discounted reward policy evaluation \citep{XuWanLia_20} (a batching method) and its multi-agent counterpart~\citep{CheZhoChe_21}, the sample complexity of local TD-update only differs by a $\log$ factor. 
We note that, in \citep{XuWanLia_20,HaiLiuLu_22,CheZhoChe_21}, the algorithms are complete actor-critic algorithms.
Thus, we only compare our results with their policy evaluation counterparts(i.e., critic steps).

{\bf 2) Communication and sample complexity in comparison with vanilla approach:} In the local TD-update algorithm, between consecutive communication rounds, the number of local TD-update steps for each agent can be $K=\mathcal{O}(1/\epsilon^{1/2}\log (1/\epsilon))$. 
This improved the communication complexity of vanilla distributed TD algorithms \citep{ZhaYanLiu_18,DoaMagRom_19,DoaMagRom_21} by a factor of $K=\mathcal{O}(1/\epsilon^{1/2}\log (1/\epsilon))$. 
The communication complexity of the local TD-update is $L=\mathcal{O}(1/\epsilon^{1/2}\log (1/\epsilon))$. In terms of sample complexity, both approaches require a sample complexity of $\mathcal{O}(1/\epsilon\log^{2} (1/\epsilon))$. This is because as we set local step $K=1$ of local TD approach, it reduces to the vanilla approach.

{\bf 3) Communication and sample complexities in comparison with batching approach:} It is worth noting that {\em ``batching''} \cite{HaiLiuLu_22} is another natural TD learning approach that can achieve infrequent communication among agents via locally updating value function parameters using a batch of $M(\ge 1)$ samples, then performing consensus. 
Specifically, instead of repeatedly updating $w^{i}$ for each sample locally as in Line 10 in Algorithm \ref{alg: local_average}, at each communication round $l\in\{0,\cdots,L-1\}$, the batching approach performs the following update:
\begin{align}
\tilde{w}^{i}_l\leftarrow w^{i}_l+\frac{1}{M}\sum_{\tau=0}^{M-1} \delta^{i}_{l,\tau}(w^{i}_l)\cdot \phi(s_{l,\tau}), \nonumber
\end{align}
which is followed by a consensus update same as Line~12 in Algorithm~\ref{alg: local_average} for $\{\tilde{w}^{i}_l\}_{i=1}^{N}$. 
The full algorithm description of the batching approach can be found in \cite[Algorithm~1]{HaiLiuLu_22}. 
The key difference between batching and local TD-update approaches is that the $w$-parameters are updated repeatedly with each sample in local TD-update, whereas in batching, the $w$-parameters are updated {\em only once} through a batch of samples.

Under the average reward setting, the local TD-update approach achieves the same communication complexity.
However, the local TD-update approach outperforms the batching approach in terms of sample complexity.
Specifically, the sample complexity upper bound of the local TD-update approach is $\mathcal{O}(1/\epsilon\log^{2} (1/\epsilon))$.
In contrast, the sample complexity of the batching approach is $\mathcal{O}(1/\epsilon^{3/2}\log (1/\epsilon))$, which is worse than that of the local TD-update approach by a factor of $\mathcal{O}(1/\epsilon^{1/2}/\log(1/\epsilon))$. 

To conclude the comparisons, we list the sample and communication complexities of different approaches in Table \ref{tab: com}.

\begin{table}[t!]
\caption{Comparison of sample and communication complexities.}
\begin{center}
\begin{tabular}{c|c|c}
\hline
Approaches & Sample Complexity & Communication Complexity  \\
\hline
Vanilla & $\mathcal{O}(1/\epsilon\log^{2} (1/\epsilon))$ & $\mathcal{O}(1/\epsilon\log^{2} (1/\epsilon))$ \\
\hline
Batching & $\mathcal{O}(1/\epsilon^{3/2}\log (1/\epsilon))$ & $\mathcal{O}(1/\epsilon^{1/2}\log(1/\epsilon))$ \\
\hline
Local TD & $\mathcal{O}(1/\epsilon\log^{2} (1/\epsilon))$  & $\mathcal{O}(1/\epsilon^{1/2}\log(1/\epsilon))$ \\
\hline
\end{tabular}
\label{tab: com}
\end{center}
\end{table}

\section{Experimental results} \label{sec:numerical}
In this section, we conduct numerical experiments to compare our proposed algorithm, TD learning with local steps, with vanilla TD learning \citep{ZhaYanLiu_18,DoaMagRom_19,DoaMagRom_21} and the batch TD learning \citep{HaiLiuLu_22,CheZhoChe_21} in both synthetic settings as in \citep{ZhaYanLiu_18} and cooperative navigation tasks as in \citep{LowWuTam_17}. 

\subsection{Performance with Synthetic Experiments}

{\bf 1) Synthetic Experiment Setup:}
We consider the same setting as in Section~6.1 of \citep{ZhaYanLiu_18}. 
There are $N=20$ agents, each of which has a binary-valued action space, i.e., $\mathcal{A}^{i}=\{0,1\}$ for all $i\in\mathcal{N}$. There are $|\mathcal{S}|=10$ states. The entries in the transition matrix are uniformly sampled from the interval $[0,1]$ and normalized to be stochastic. For each agent $i$ and global state action pair $(s,a)$, the reward $r^{i}(s,a)$ is sampled uniformly from $[0,4]$ and the instantaneous rewards $\{r^{i}_t\}$ are sampled uniformly within the set $[r^{i}(s,a)-0.5,r^{i}(s,a)+0.5]$. The policy considered in the simulation is $\pi^{i}(\cdot|s)=0.5$ for all $i\in\mathcal{N}$, $s\in\mathcal{S}$. 
The entries of feature matrix $\Phi$ are sampled uniformly at random from $[0,1]$ with feature dimension $n=5$ and ensured to be full rank and satisfy Assumption \ref{ass: fea}. In addition, we set each feature vector to be of unit length.
The network topology is chosen as a ring network with diagonal elements being 0.4 and off-diagonal elements being 0.3. 
The simulation results are averaged over 10 trials. 
We choose the step sizes for our algorithm to be 0.005, vanilla TD to be 0.1, and batch TD to be 0.1. We note that these step sizes are chosen to be best for the corresponding algorithms. 

The objective error is defined as the normalized version of convergence term (LHS of Eq.~(\ref{eq: convergence})), i.e., the sample mean errors divided by the number of agents $N$ and the dimension number $n$:
\begin{align}
    &\text{Objective Error} \nonumber \\
    &:=\text{sample average of } \frac{\sqrt{\sum_{i=1}^{N}\|w^{i}_{l,k}-w^{*}\|^{2}}}{nN} \text{ for 10 trials}. \nonumber
\end{align}

We remark that due to the fact that the transition matrix is not dependent on joint action, the steady state distribution can be computed and so is the value of $w^{*}$, whose definition is in Eq.~(\ref{eq: ode}).

\begin{figure}[h!]
     \begin{subfigure}[b]{0.23\textwidth}
         \centering
         \includegraphics[width=\textwidth]{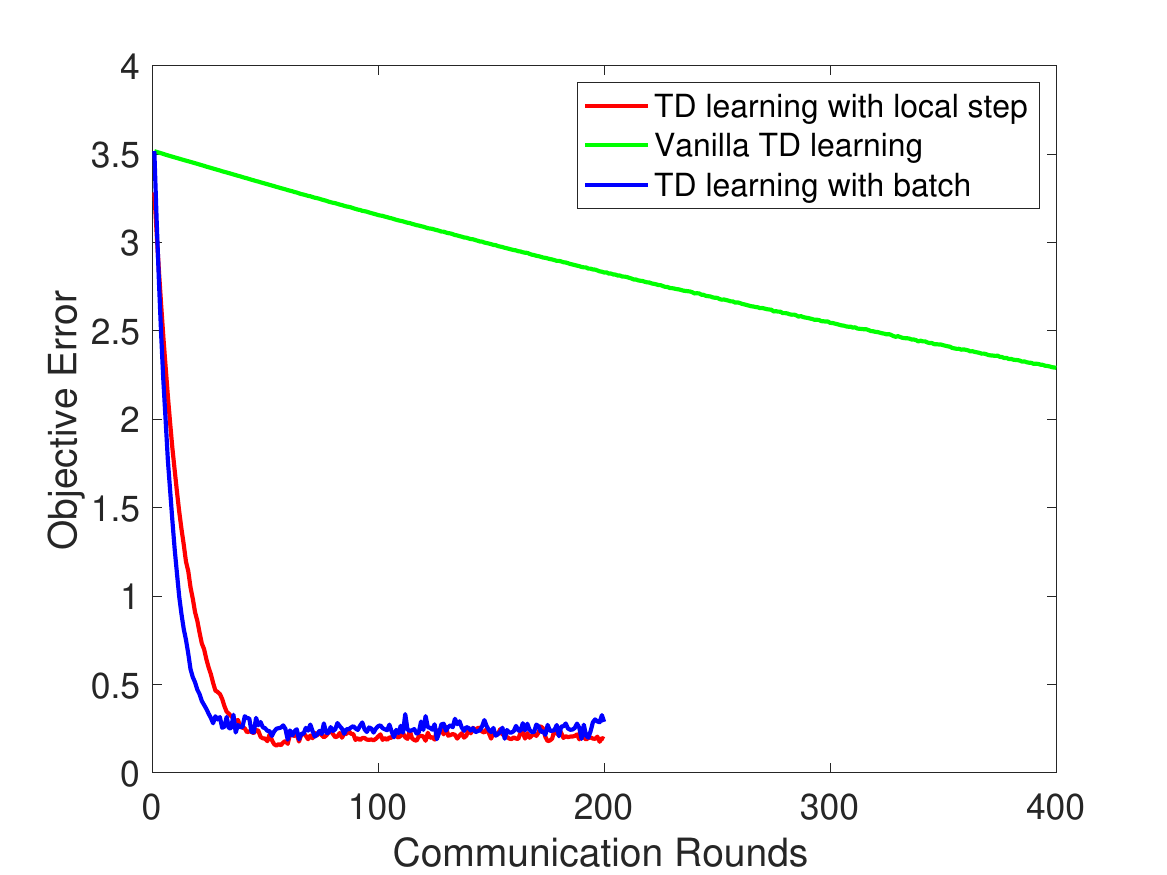}
         \caption{$K=50,L=200$.}
         \label{fig: com_1a}
     \end{subfigure}
     \hfill
     \begin{subfigure}[b]{0.23\textwidth}
         \centering
         \includegraphics[width=\textwidth]{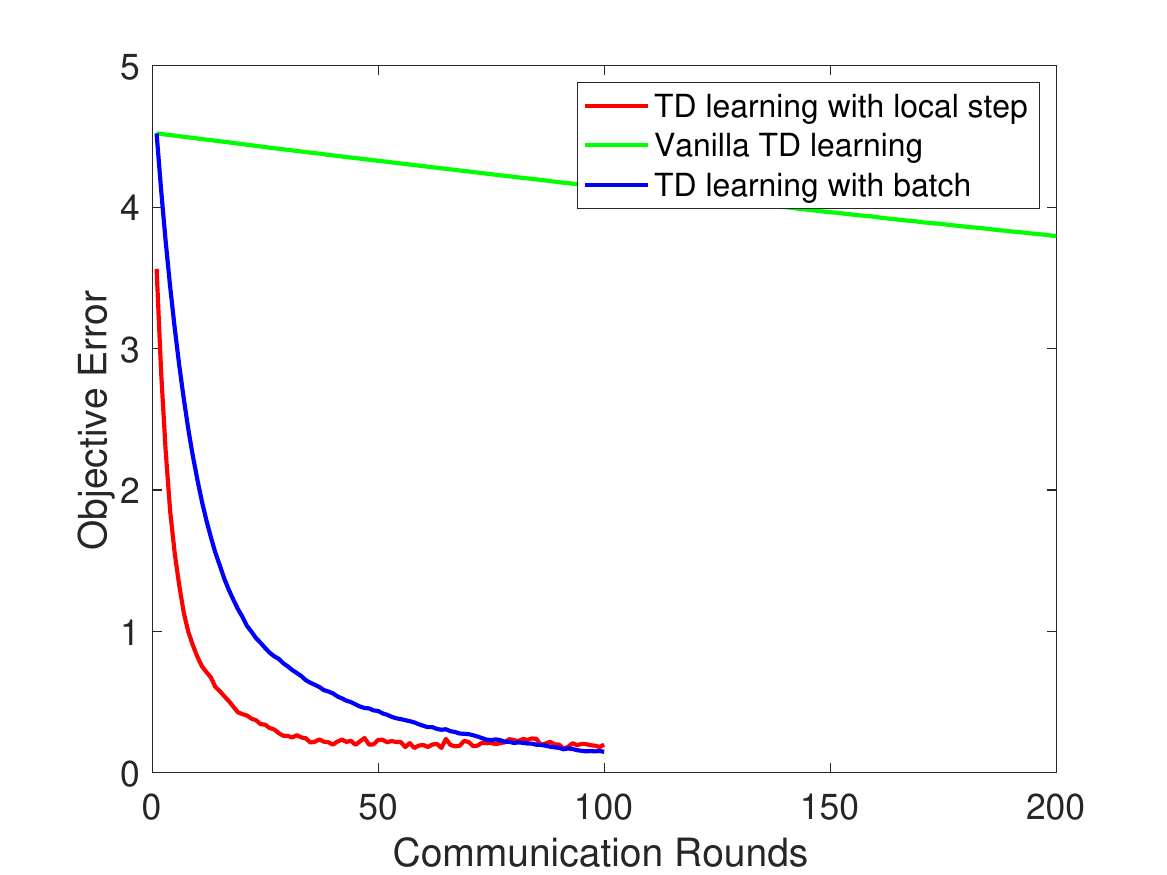}
         \caption{$K=100,L=100$.}
         \label{fig: com_1b}
     \end{subfigure}
     \hfill
    \begin{subfigure}[b]{0.23\textwidth}
         \centering
         \includegraphics[width=\textwidth]{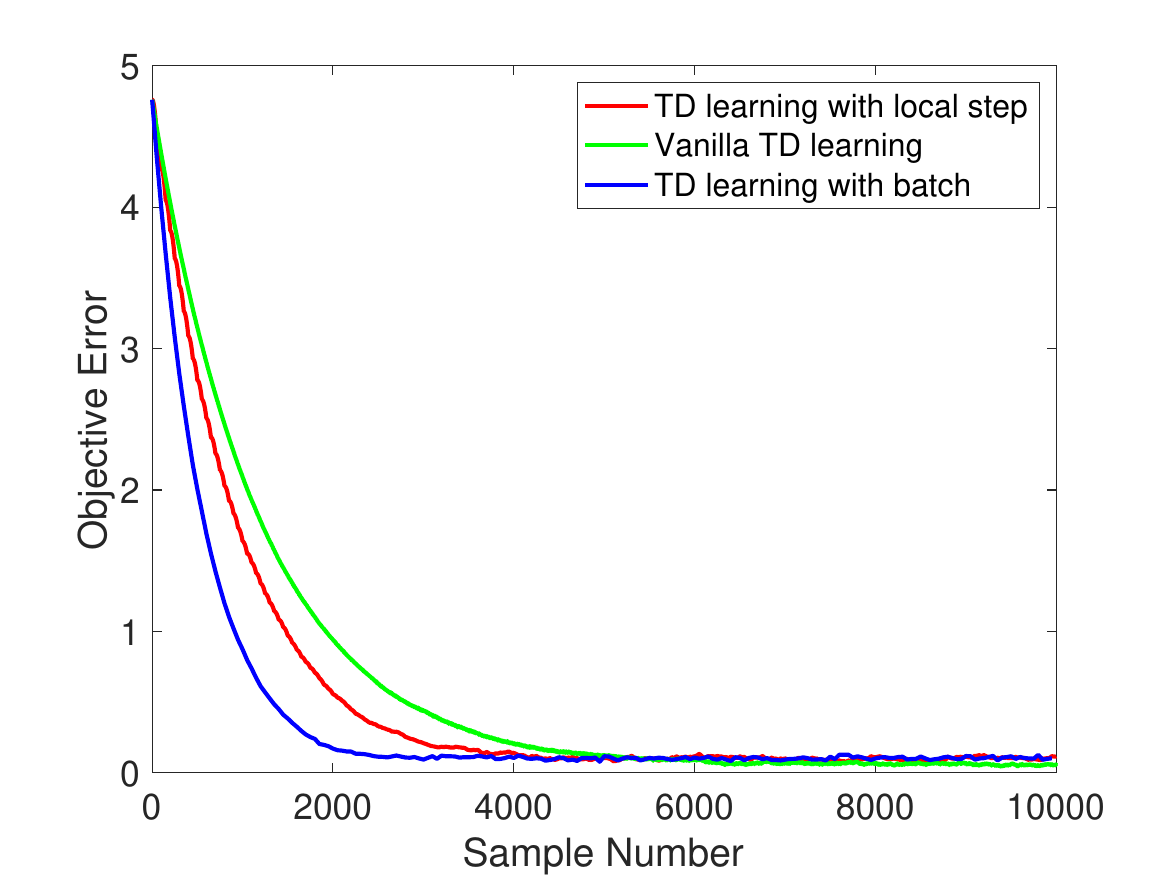}
         \caption{$K=50,L=200$.}
         \label{fig: L200}
     \end{subfigure}
     \hfill   
     \begin{subfigure}[b]{0.23\textwidth}
         \centering
         \includegraphics[width=\textwidth]{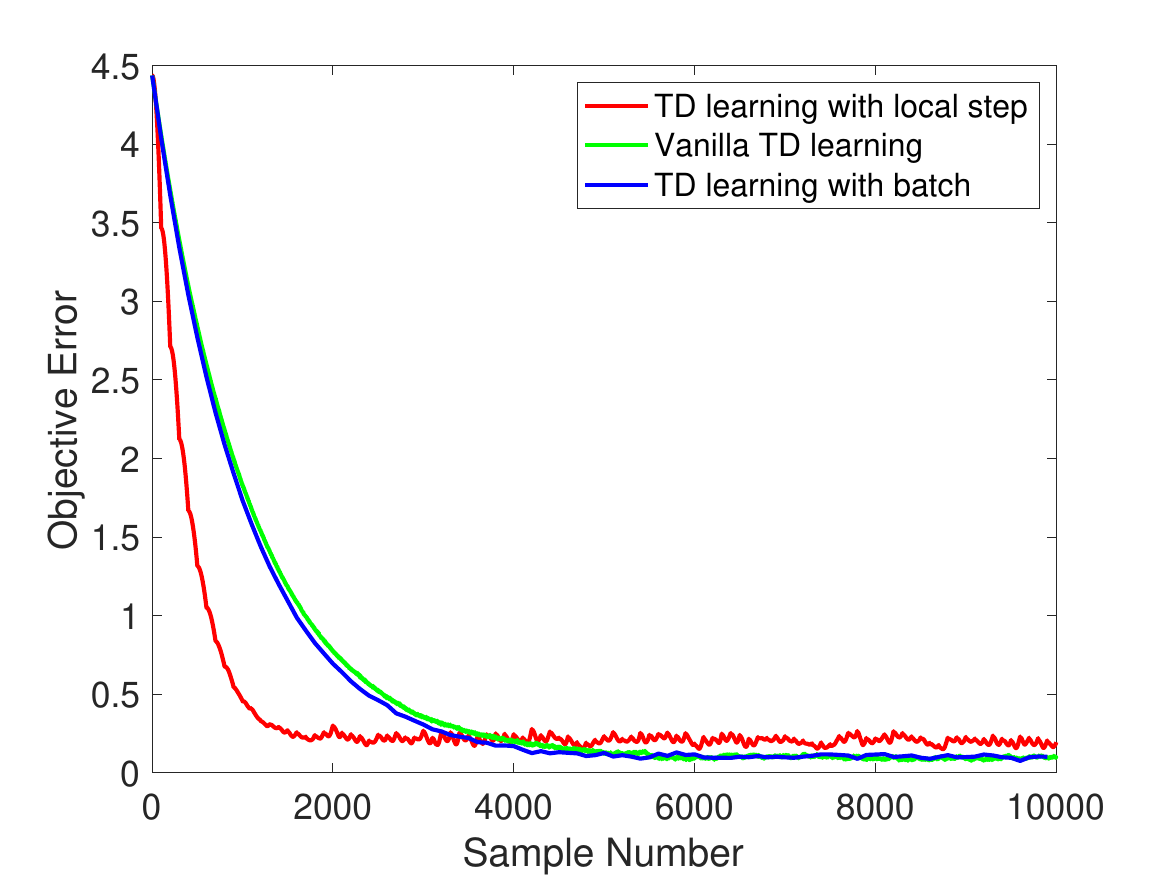}
         \caption{$K=100,L=100$.}
         \label{fig: L100}
     \end{subfigure}
        \caption{Convergence with respect to the number of communication rounds and samples.}
        \label{fig: comp}
        \vspace{-.1in}
\end{figure}

\smallskip
{\bf 2) Convergence Performance:}
In Fig.~\ref{fig: comp}, the y-axis is the normalized convergence error of the LHS of Eq.~(\ref{eq: convergence}) and the x-axes are the numbers of communication rounds in Figure \ref{fig: com_1a},\ref{fig: com_1b} and sample numbers in Figs.~\ref{fig: L200} and \ref{fig: L100}. 
For fair comparisons between the local TD-update and batching approaches, we keep the local TD-update step number and batch size to be the same for the majority of the comparisons except for Fig.~\ref{fig: add_comp}, where we compare the results for various local TD-update step numbers and batch sizes.

In Fig.~\ref{fig: com_1a}, we illustrate the convergence results with respect to the communication rounds for all three algorithms, where the local TD-update step $K=50$ for the local TD-update approach and the batch size is 50 for batch algorithm. 
Under such a setting, both local TD-update and batched TD algorithms perform consensus communication every 50 samples.  
We can see that within 200 communication rounds, both local TD-update and batching algorithms converge to a very similar error level, yet the vanilla TD algorithm does not converge even after 400 rounds of communication. 
Between local TD-update and batching, both algorithms perform similarly, which means similar communication rounds to converge. 
In Fig.~\ref{fig: com_1b}, when local TD-update step $K=100$ and the batch size is $100$, the local TD-update approach requires the least amount of communication rounds to converge compared to the batching approach. 
On the other hand, local TD-update again performs significantly better compared to vanilla TD. 
In Fig.~\ref{fig: L100} and \ref{fig: L200}, we illustrate the corresponding convergence results with respect to the number of samples. 
We can see that vanilla TD eventually converges but requires consensus operation at every sample. 
Fig.~\ref{fig: comp} verifies the theoretical analysis that allowing local TD-update steps does reduce the number of communication rounds compared to vanilla TD. 
In addition, the communication rounds of local TD-update algorithm is similar in the setting of Fig.~\ref{fig: com_1a} and significantly better in the setting of Fig.~\ref{fig: com_1b}. 

In addition, we compare local TD-update approach under different number of local TD-update steps $K$ and communication rounds $L$ with batching approach under different batch sizes $M$ and communication rounds $L$ in Fig.~\ref{fig: add_comp}. 
In general, the local TD-update approach converges faster than the batching approach, but with a slightly larger objective error. 
As the number of local TD-update steps increases, the convergence speed of the local TD-update approach converges also increases, but the objective error becomes larger.
This verifies the ``agent-drift'' phenomenon.
In contrast, as the batch size increases, the convergence speed becomes slower, and the objective error continues to improve.

\begin{figure}[t!]
     \begin{subfigure}[b]{0.23\textwidth}
         \centering
         \includegraphics[width=\textwidth]{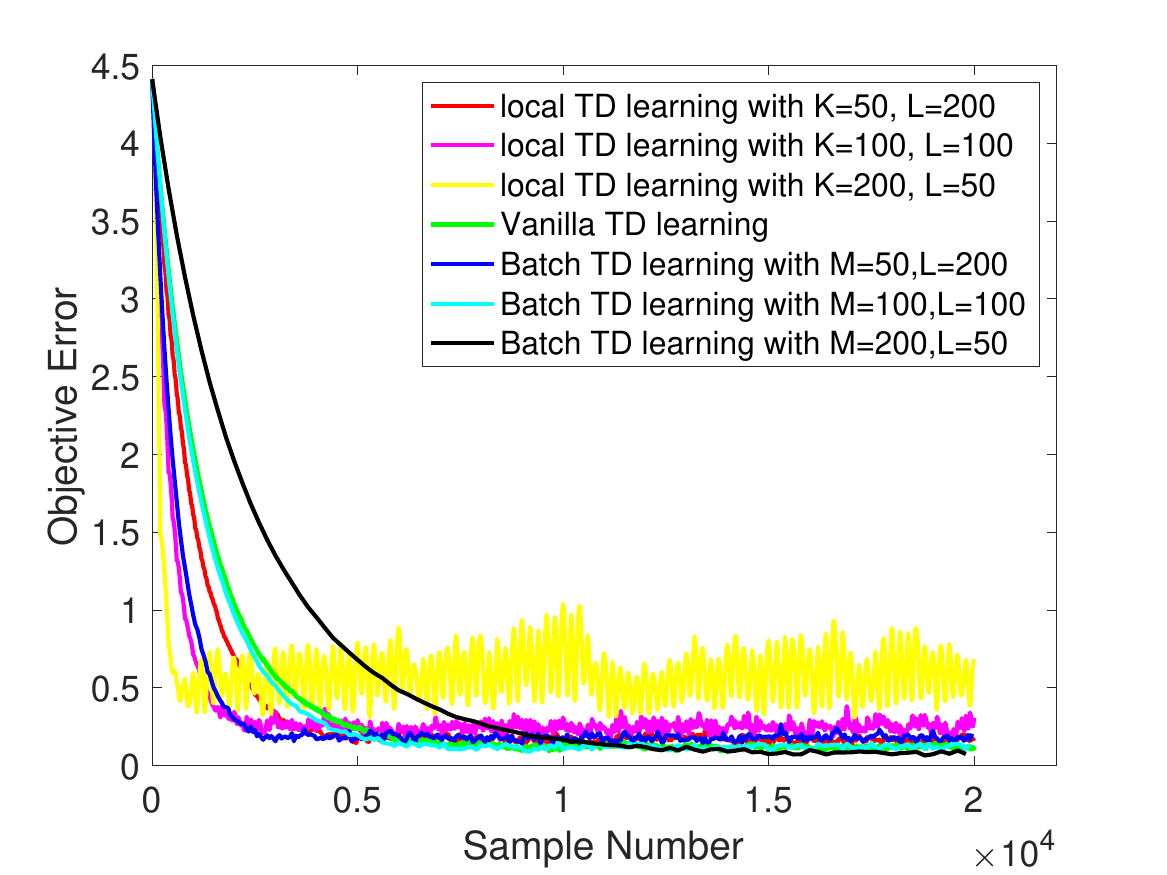}
         \caption{Different $(K,M,L)$ values.}
         \label{fig: add_comp}
     \end{subfigure}
     \begin{subfigure}[b]{0.23\textwidth}
         \centering
         \includegraphics[width=\textwidth]{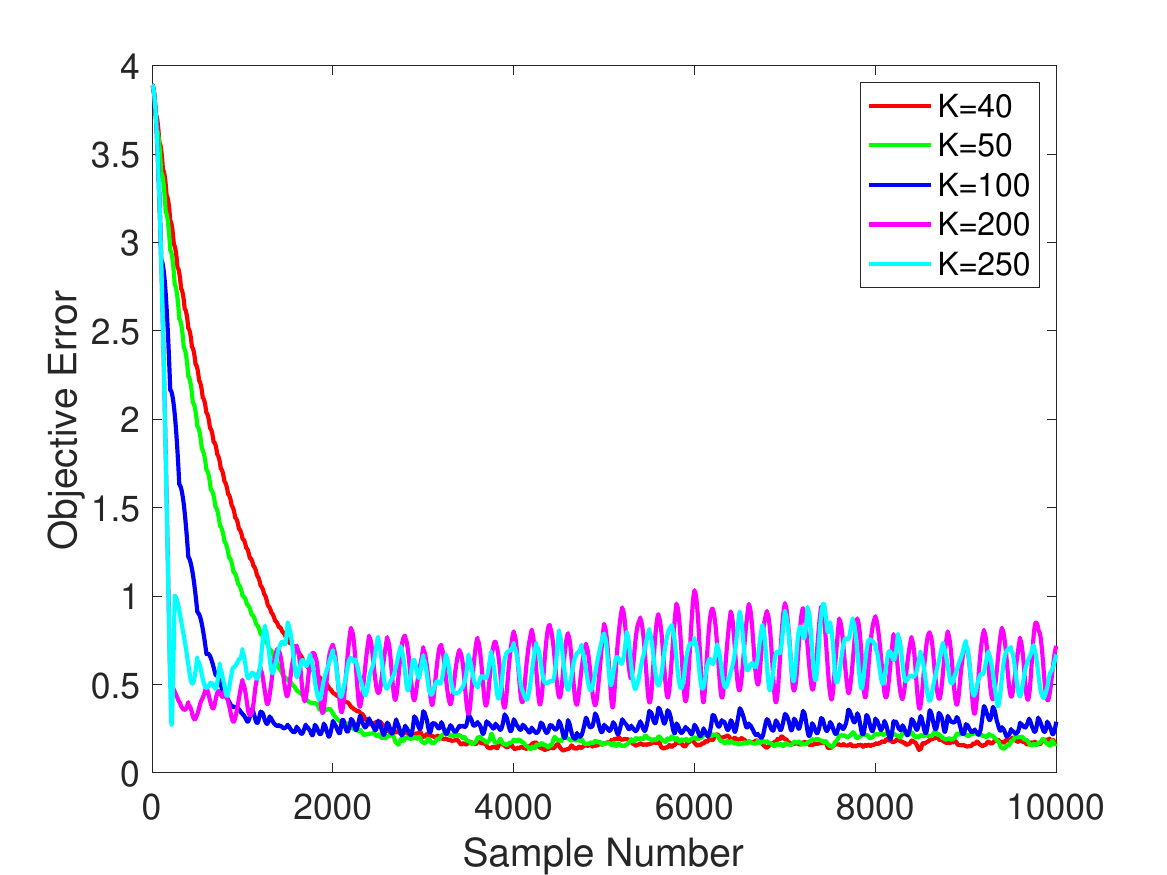}
         \caption{The impacts of $K$.}
         \label{fig: local_steps}
     \end{subfigure}
     \caption{Convergence comparisons with different settings of $(K,L)$ and the impact of local TD-update steps $K$ on convergence performance.}
     \vspace{-.1in}
     \centering
\end{figure}

\smallskip
{\bf 3) Impacts of the Number of Local TD-Updates:}
Next, we further investigate the effect of the number of local TD-update steps on the convergence of the local TD-update approaches and the agent-drift phenomenon. 
In Fig.~\ref{fig: local_steps}, we vary the number of local steps from $K=40$ to $K=250$. 
There are two interesting observations from our experiments. 
First, the initial dropping of objective error increases as the number of local TD-update steps increases. 
For example, when $K=100$ or larger, the curves drop much more rapidly in the beginning compared to the curves with a smaller $K$. 
Second, the objective error floor increases as the number of local steps increases. 
For example, when $K\le 100$, the objective error floor is relatively low and stable.
However, as $K$ increases to 200 or 250, the objective error floor also increases with a larger oscillation magnitude. 
This observation is consistent with our theoretical analysis in Lemma~\ref{lem: con_err}, where the second term on the RHS of Eq.~(\ref{eq: con_err}) is proportional to the product of step size $\beta$ and local TD-update step $K$. 
This term indicates that the objective error will only converge to neighborhood of zero, whose size depends on $\beta K$. 
As a result, for a larger $K$-value, the objective error will oscillate with a larger magnitude. 
This is similar to the constant error term in the convergence of the dencentralized SGD method~\citep{NedOzd_09}.
Also, the agent-drift phenomenon worsens as the number of local TD-update steps increases, which can be seen by the result of $K \ge 200$ in Fig.~\ref{fig: local_steps}.
To summarize, under a fixed step size, more local TD-update steps improve the initial convergence speed, but will eventually result in a larger objective error floor. 

\begin{figure}[t!]
  \centering
  \vspace{-.1in}
  \includegraphics[width=.2\textwidth]{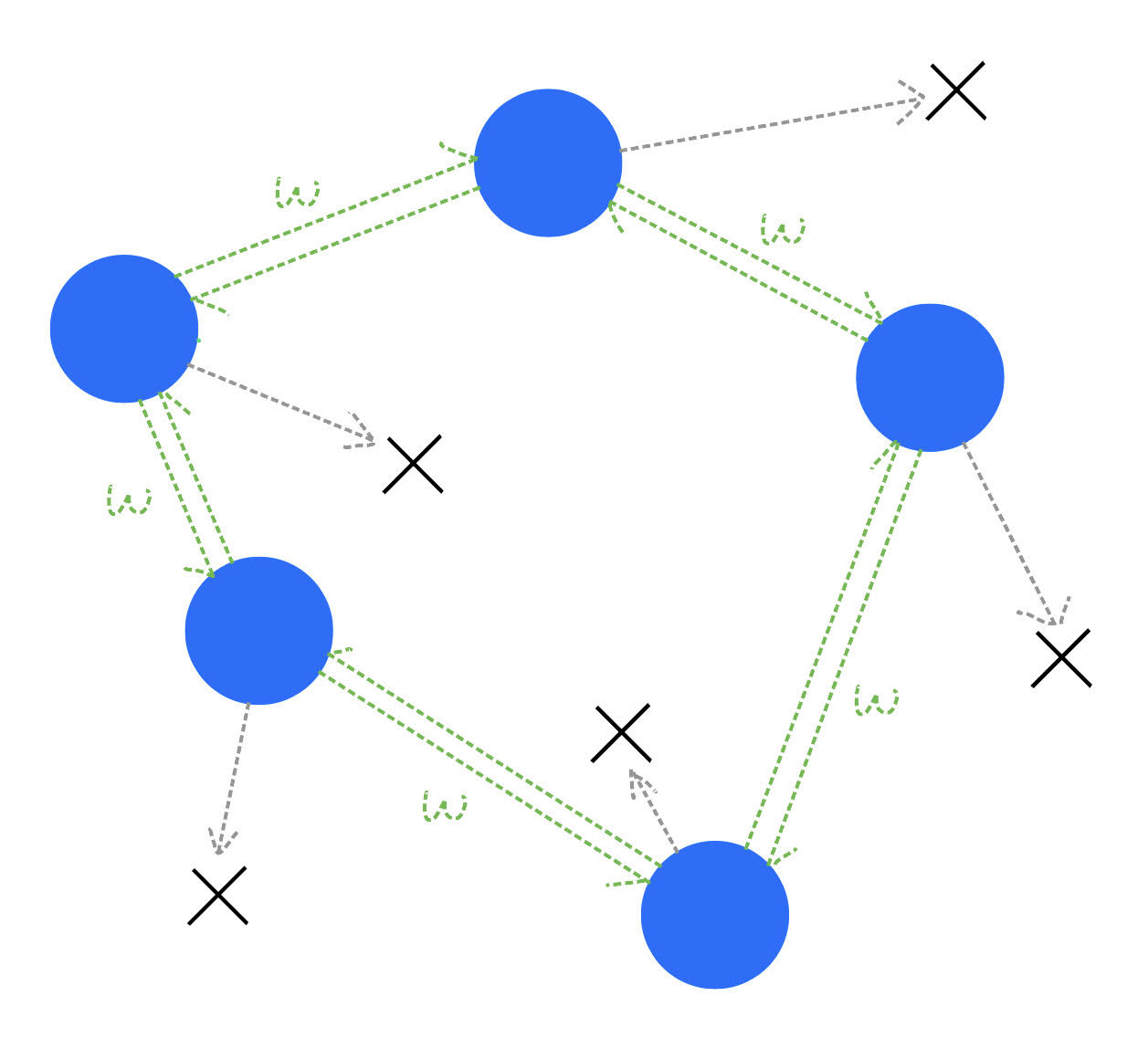}
  \caption{A cooperative navigation task.}
  \label{fig: task}
  \vspace{-.2in}
\end{figure}

\subsection{Performance with Cooperative Navigation}
\label{sec: coop_navi}
As illustrated in Fig.~\ref{fig: task}, in the cooperative navigation task~\cite{LowWuTam_17,ZhaYanLiu_18}, the agents (blue circles) are trained to cover the landmarks (crosses). 
Agents observe positions of all other agents and all landmarks and collaboratively cover the landmarks while avoiding collisions. 
The rewards for agents are defined through the proximity to the nearest landmarks. 
Unlike the synthetic experiments, the fixed point of the corresponding ODE as in Eq.~(\ref{eq: ode}) is difficult to compute.
Thus, we use the mean squared Bellman error(MSBE) as the performance metric. 
Due to space limitation, we relegate some experimental results to our online technical report~\cite{HaiZhaLiu_23}, including discussions on various network typologies, local TD-update steps, batch sizes, step sizes, and consensus error metrics.

\smallskip
{\bf 1) Experiment Setup and Performance Metrics:}
We consider a cooperative navigation task that is adapted from one of the multi-agent environments~\citep{LowWuTam_17}. 
There are \textit{N} = 9 agents in total, and the goal is to cover 9 landmarks collaboratively. 
Each agent chooses from the action space $\mathcal{A}^{i}=$\{no action, move left, move right, move down, move up\} based on the given policy $\pi$. 
The policy considered in the simulation is $\pi^{i}(\cdot|s)=0.2$ for all actions and $i\in\mathcal{N}$, $s\in\mathcal{S}$, i.e. uniformly random policy. 
The local rewards are given by the distance between the agents and the nearest goal landmarks. 
However, if the agents collide with each other, a penalty will incur. 
The agents are trained to cover landmarks and reach the destination, while avoiding to collide with other agents, and the entire learning process is fully decentralized. 
The feature dimension here is $n=36$, which includes all agents' self positions, landmark relative positions, and other agent relative positions. 
We choose step sizes for the TD-update and the vanilla TD approaches to be both $0.1$. 
We note that such step sizes are chosen for the best performance for the corresponding algorithms.

As mentioned earlier, we adopt the mean squared Bellman error (MSBE) as our performance metric. 
Given $w$-parameters and samples $(s_k,s_{k+1})$, the empirical squared Bellman error (SBE) of the $\kappa$-th sample is defined as:
\begin{align}
&\text{SBE}\left(\left\{w_{\kappa}^i\right\}_{i=1}^{N}, s_{\kappa},s_{\kappa + 1}\right) \nonumber \\
:&=\frac{1}{N} \sum_{i \in \mathcal{N}}\left(\phi(s_{\kappa})^{T}w^{i}_{\kappa}+\bar{\mu}_{\kappa}-\bar{r}_{\kappa}- \phi(s_{\kappa+1})^{T}w^{i}_{\kappa}\right)^2, \nonumber
\end{align}
where $\bar{r}_{\kappa}=\frac{1}{N}\sum_{i\in\mathcal{N}} r^{i}_\kappa$ and $\bar{\mu}_\kappa=\frac{1}{N}\sum_{i\in\mathcal{N}} \mu^{i}_\kappa$.
Then, MSBE up to the $k$-th sample is defined as the average of SBEs over the history, which is as follows:
$$\text{MSBE}:=\frac{1}{k} \sum_{\kappa=1}^k \text{SBE}\left(\left\{w^i_\kappa\right\}_{i=1}^{N}, s_\kappa,s_{\kappa+1}\right).$$

\begin{figure}[h!]
     \begin{subfigure}[b]{0.23\textwidth}
         \includegraphics[width=\textwidth]{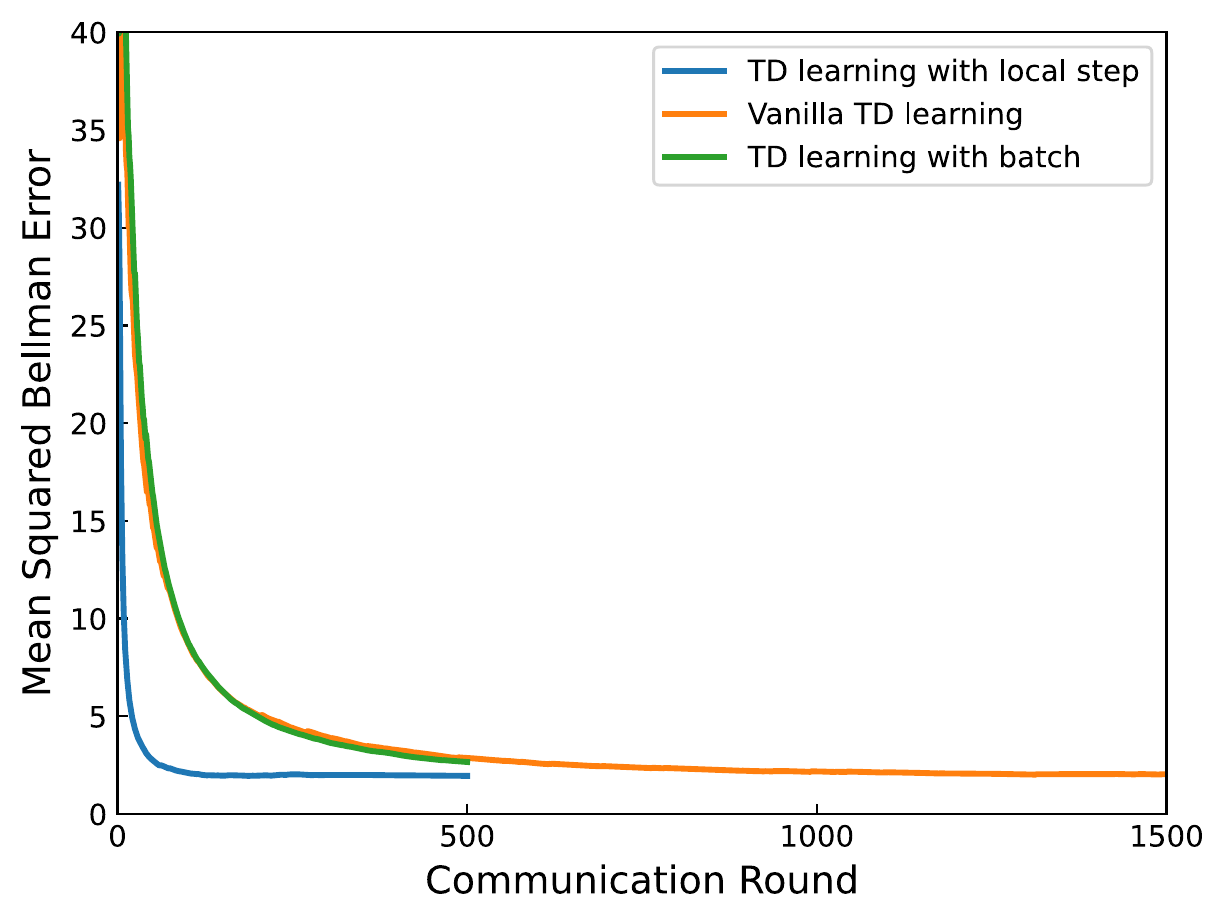}
         \caption{Bellman error with respect to communication rounds.}
         \label{fig: bell_com}
     \end{subfigure}
     \begin{subfigure}[b]{0.23\textwidth}
         \includegraphics[width=\textwidth]{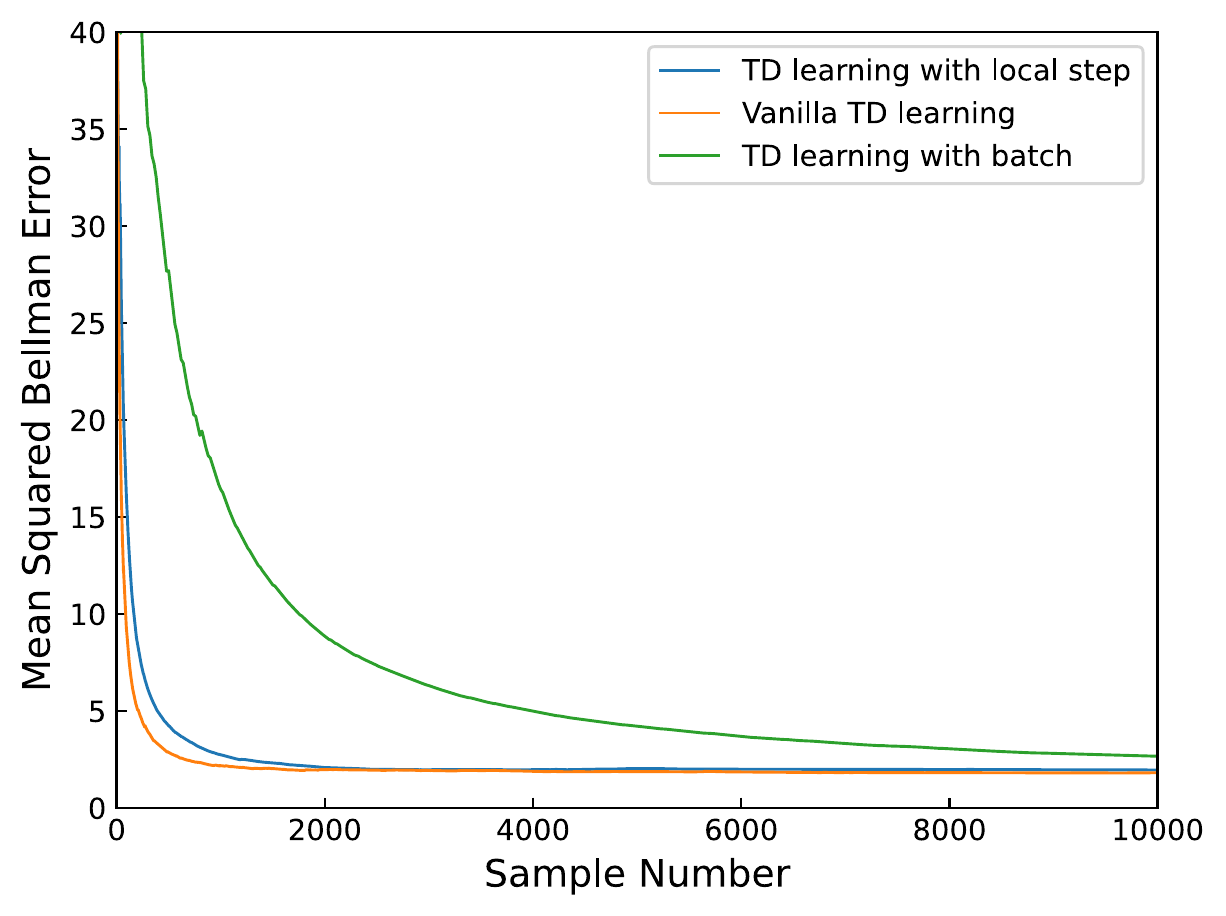}
         \caption{Bellman error with respect to the number of samples.}
         \label{fig: bell_samp}
     \end{subfigure}
     \caption{Convergence in terms of the number of communication rounds and training samples.}
     \label{fig: bell}
     \vspace{-.1in}
\end{figure}

\smallskip
{\bf 2) Convergence Performance:}
In Fig.~\ref{fig: bell_com} and \ref{fig: bell}, we illustrate the results of MSBEs with respect to the number of communication rounds and training samples, where $N=9$ agents are connected through an Erdos-Renyi (ER) network. 
We set the number of local TD-update steps and the batch size both to be $20$ for the local TD-update and batching approaches, respectively. 
Similar to the synthetic experiments, all algorithms converge to similar levels of MSBE as shown in Fig.~\ref{fig: bell}. 
This again verifies our theoretical analysis that allowing local TD-update steps and performing infrequent communications do not affect convergence. 
Moreover, in this setting, the local TD-update algorithm converges much faster in terms of the number of communication rounds. 
Specifically, in Fig.~\ref{fig: bell_com}, the local TD-update algorithm requires roughly 250 rounds of communication to converge, while both the batching and vanilla TD algorithms perform similarly and require more than 500 rounds of communication to converge.

\section{Conclusion} \label{sec:conclusion}
In this paper, we investigated the question of whether the local TD-update approach can achieve low sample and communication complexities for multi-agent reinforcement learning policy evaluation (MARL-PE) under the average reward setting and, if so, how is the performance in comparison with other approaches under the average reward setting. 
Our theoretical analysis and experimental results show that the local TD-update approach can significantly lower the communication complexity compared to the vanilla TD learning. 
In addition, our theoretical analysis also shows that the number of local TD-update steps can be as large as $K=\mathcal{O}(1/\epsilon^{1/2}\log(1/\epsilon))$ to converge to an $\epsilon$-neighborhood of the solution of the corresponding ODE for MARL-PE. 
Compared with the batching approach for solving MARL-PE under average reward, the local TD-update approach achieves the same communication complexity as that of the batching approach, while enjoying a better sample complexity by a factor of $\mathcal{O}(1/\epsilon^{1/2})$ than that of the batching approach. 
Our experimental results also verify our theoretical findings in both synthetic and real-world data settings.

\appendix
\section{Single-Agent Policy Evaluation Convergence under the Average Reward Setting}

In this section, we provide the finite-time convergence result for single-agent RL in the average reward setting, as the update of average parameter $\bar{w}$ in Lemma \ref{lem: ave_con} is essentially a centralized single-agent TD learning.
The finite time convergence for a more general form of stochastic approximation has been established in \cite{SriYin_19}. We utilize such results by verifying the conditions in \cite{SriYin_19}. 
\subsection{Single Agent RL in Average Reward Setting}
We first describe the single agent TD(0) algorithm in the average reward setting in Algorithm \ref{alg: sa_average}.

\begin{algorithm}[t]
  \SetKwInOut{Input}{Input}
  \SetKwInOut{Output}{Output}
  \Input{Initial state $s_0,\pi$, feature map $\phi$, initial parameters $w_{0}, \mu_0$, step size $\beta$, traning iteration $T$}
  \BlankLine
  \For{$t=0,\cdots,T-1$}{
	Execute action $a_{t}\sim\pi(\cdot|s_t)$\;
	Observe the state $s_{t+1}$ and reward $r_{t+1}$\;
	Update $\delta_{t}\leftarrow r_{t+1}-\mu_{t}+\phi(s_{t+1})^{T}w_{t}-\phi(s_t)^{T}w_{t}$\;
 	Update $\mu_{t+1}\leftarrow \beta r_{t+1}+(1-\beta)\mu_{t}$\;
	TD Step: $w_{t+1}\leftarrow w_{t}+\beta\delta_{t}\cdot\phi(s_t)$\;
  }
  \Output{$w_T$}
  \caption{Single Agent TD(0) Learning in Average Reward Setting}
  \label{alg: sa_average}
\end{algorithm}

\begin{theorem}
Suppose $N=1$ and Assumptions \ref{ass: dis}-\ref{ass: fea} hold. For the parameter generated by Algorithm \ref{alg: sa_average}, we have following results: 
\begin{align}
\mathbb{E}[\|w_T-w^{*}\|^{2}]\le & c_2(1-c_1\beta)^{T-\tau(\beta)}(\sqrt{\|w_{0}-w^{*}\|^{2}+(\mu_0-J_{\pi})^{2}} \nonumber \\
& +\frac{r_{\max}}{3})^{2}+c_3\beta \tau(\beta),
\label{eq: single_agent}
\end{align}
where $c_1, c_2, c_3>0$ are constants that are independent of step size $\beta$ and iteration number $T$; and $\tau(\beta)=\mathcal{O}(\log \frac{1}{\beta})$ is the mixing time.
\label{thm: single_agent}
\end{theorem}

\begin{proof}
To apply Theorem~7 in \cite{SriYin_19}, we need to verify the three conditions in \cite[Section~2.1]{SriYin_19}. We have the following notations 
$Z_t=(s_t,a_t)$ and $r_{t+1}=r(Z_t)$. For TD(0) learning under the average reward setting, we have 
\begin{align}
\mu_{t+1}&=\mu_t+\beta (r_{t+1}-\mu_t), \nonumber \\
w_{t+1}&=w_t+\beta (r_{t+1}-\mu_t+\phi^{T}(Z_{t+1})w_t-\phi^{T}(Z_t)w_t)\phi(Z_t). \nonumber
\end{align}
Equivalently, the matrix form is 
\begin{align}
{\begin{pmatrix}
\mu_{t+1} \\[\jot]
w_{t+1}
\end{pmatrix}}={\begin{pmatrix}
\mu_{t} \\[\jot]
w_{t}
\end{pmatrix}}+&\beta
\begin{pmatrix}
 -1 & 0 \\[\jot] -\phi(Z_t) & \phi(Z_t)(\phi(Z_{t+1})-\phi(Z_t))^{T}
\end{pmatrix}\cdot\begin{pmatrix}
 \mu_t \\[\jot] w_t
\end{pmatrix} \nonumber \\
+&\beta\begin{pmatrix}
r_{t+1} \\[\jot] \phi(Z_t)r_{t+1}
\end{pmatrix}.
\label{eq: sde}
\end{align}
So the corresponding ODE can be written as:
\begin{align}
{\begin{pmatrix}
\dot{\mu} \\[\jot]
\dot{w}
\end{pmatrix}}=\begin{pmatrix}
 -1 & 0 \\[\jot] -\Phi^{T}D^{s}\mathbf{1} & \Phi^{T}D^{s}(P^{\pi}-I)\Phi
\end{pmatrix}\cdot\begin{pmatrix}
 \mu \\[\jot] w
\end{pmatrix}+\begin{pmatrix}
J_{\pi} \\[\jot] \Phi^{T}D^{s}\bar{R}
\end{pmatrix},
\label{eq: ODE}
\end{align}
where $P^{\pi}$ is the state transition matrix induced by the policy $\pi$, $D^{s}=\text{diag}(d(s_1),\cdots, d(s_{|\mathcal{S}|}))$ and $\bar{R}:=[\bar{R}(s),s\in\mathcal{S}]^{T}$, where $\bar{R}(s)=\sum_{a}\pi(a|s)r(s,a)$.
Now, using the notation in \cite{SriYin_19}, we have
\begin{align}
\bar{A}=\tilde{A}=\begin{pmatrix}
 -1 & 0 \\[\jot] -\Phi^{T}D^{s}\mathbf{1} & \Phi^{T}D^{s}(P^{\pi}-I)\Phi
\end{pmatrix} \nonumber 
\end{align}
and 
\begin{align}
\tilde{b}=\begin{pmatrix}
J_{\pi} \\[\jot] \Phi^{T}D^{s}\bar{R}
\end{pmatrix}. \nonumber
\end{align}
Next, by centering 
$\begin{pmatrix}
 \mu \\[\jot] w
\end{pmatrix}\leftarrow \begin{pmatrix}
 \mu \\[\jot] w
\end{pmatrix}-\begin{pmatrix}
J_{\pi} \\[\jot] w^{*}
\end{pmatrix}$ and defining, we have 
\begin{align}
X_k&=(Z_k,Z_{k+1})^{T}, \nonumber \\
A(X_k)&=\begin{pmatrix}
 -1 & 0 \\[\jot] -\phi(Z_t) & \phi(Z_t)(\phi(Z_{t+1})-\phi(Z_t))^{T}
\end{pmatrix}, \nonumber \\
b(X_k)&=\begin{pmatrix}
r_{t+1} \\[\jot] \phi(Z_t)r_{t+1}
\end{pmatrix}-A(X_k)\cdot\begin{pmatrix}
J_{\pi} \\[\jot] w^{*}
\end{pmatrix}, \nonumber \\
\bar{b}&=0. \nonumber
\end{align}

Note that $\bar{A}\begin{pmatrix}
J_{\pi} \\[\jot] w^{*}
\end{pmatrix}=\tilde{b}$.
Next, consider the following conditions:
\begin{itemize}
\item \textbf{Condition 1:} Note that 
\begin{align}
&||E[b(X_k|X_0=(Z_0,Z_1)=(z_0,z_1))]|| \nonumber \\
=&||\sum_{i}(P(Z_k=i|(Z_0,Z_1)=(z_0,z_1))-d(i)) \nonumber \\
&\cdot\left( \begin{pmatrix}
\bar{r}(i) \\[\jot] \phi(i)\bar{r}(i)
\end{pmatrix}
-\begin{pmatrix}
 -1 & 0 \\[\jot] -\phi(i) & \phi(i)(\sum_{j}p^{\pi}_{ij}\phi(j)-\phi(i))^{T}
\end{pmatrix}
\begin{pmatrix}
J_{\pi} \\[\jot] w^{*}
\end{pmatrix}\right)|| \nonumber \\
\le&||\sum_{i}(P(Z_k=i|(Z_0,Z_1)=(z_0,z_1))-d(i))||\cdot
b_{\max}, \nonumber 
\end{align}
where $b_{\max}=2(r_{\max}+J_{\pi})+2w^{*}$, where we used Assumption~\ref{ass: fea} and 
\begin{align}
&||\bar{A}-E[A(X_k)|X_0=(Z_0,Z_1)=(z_0,z_1)]|| \nonumber \\
=&||\sum_{i}(P(Z_k=i|(Z_0,Z_1)=(z_0,z_1))-d(i)) \nonumber \\
&\cdot\begin{pmatrix}
 -1 & 0 \\[\jot] -\phi(i) & \phi(i)(\sum_{j}p^{\pi}_{ij}\phi(j)-\phi(i))^{T}
\end{pmatrix}|| \nonumber \\
&\le 4||\sum_{i}(P(Z_k=i|(Z_0,Z_1)=(z_0,z_1))-d(i))||. \nonumber
\end{align}
Since $\{Z_k\}$ is a finite state, aperiodic and irreducible Markov chain, it has a geometric mixing rate, so Assumption \ref{ass: dis} holds.

\item \textbf{Condition 2:} By Assumption \ref{ass: fea}, $\max_{i\in \mathcal{S}}||\phi(i)||\le 1<\infty$ and $\max_{i\in \mathcal{S}\times\mathcal{A}}r(i)=r_{\max}$, it implies
\begin{align}
||A(X_k)||&=||\begin{pmatrix}
 -1 & 0 \\[\jot] -\phi(i) & \phi(i)(\phi(j)-\phi(i))^{T}
\end{pmatrix}||\le 4. \nonumber
\end{align}
Hence it is bounded. 
To normalize, we can set $A(i)\leftarrow \frac{A(i)}{4}$ and $b(i)\leftarrow \frac{b(i)}{4}$ to ensure the $||\bar{A}||\le 1$.
\item \textbf{Condition 3:} By the standard assumptions on the feature vectors in \cite{TsiVan_99}, we have that (1) $\Phi$ is full rank; (2) for every $\nu\in R^{n}$, $\Phi\nu\neq \mathbf{1}$. This ensures that real parts of all eigenvalues of $\bar{A}$ are strictly negative.
\end{itemize}
As a result, by directly applying Theorem 7 of \cite{SriYin_19}, we have 
\begin{align}
&\mathbb{E}[\|w_T-w^{*}\|^{2}+|\mu_T-J_{\pi}|^{2}] \nonumber \\
\le & c_2(1-c_1\beta)^{T-\tau(\beta)}(\sqrt{\|w_{0}-w^{*}\|^{2}+(\mu_0-J_{\pi})^{2}} \nonumber \\
& +\frac{r_{\max}}{3})^{2}+c_3\beta \tau(\beta), \nonumber
\end{align}
from which the result in Eq.~(\ref{eq: single_agent}) follows. 
\end{proof}

\subsection{Details of the Constants $c_1$, $c_2$, and $c_3$ in Lemma~\ref{lem: ave_con}}
\label{sec: lemma_2}
The average parameter is signaled by the average of the rewards, i.e. $\bar{r}_{l,k}=\frac{1}{N}\sum_{i\in\mathcal{N}}r^{i}_{l,k}$ during both the local TD-update and consensus steps. 
Therefore, the method updates similar to a centralized TD learning in a single-agent setting.
By applying Theorem \ref{thm: single_agent} for the average parameter, we have following results:
\begin{align}
\mathbb{E}[||\bar{w}_{L,0}-w^{*}||^{2}]\le &c_2(1-c_1\beta)^{KL-\tau(\beta)}(\sqrt{||\bar{w}_0-w^{*}||^{2}+(\mu_0-J_{\pi})^{2}} \nonumber \\
&+\frac{r_{\max}}{3})^{2}+c_3\beta \tau(\beta),
\nonumber
\end{align}
where $\tau(\beta)$ is a mixing time. Under Assumption \ref{ass: dis}, $\tau(\beta)=O(\log\frac{1}{\beta})$. 
To specify the constants $c_1, c_2, c_3$, recall the definition of $\Psi$ in Eq.~(\ref{eq: Psi_def}), which is negative definite \cite{TsiVan_99}. 
Further, define
\begin{align}
\tilde{\Psi}:=\begin{pmatrix}
 -1 & 0 \\[\jot] -\Phi^{T}D^{s}\mathbf{1} & \Psi
\end{pmatrix}, \nonumber
\end{align}
where $D^{s}=\text{diag}(d(s_1),\cdots, d(s_{|\mathcal{S}|}))$ and $\Phi$ is the feature matrix.
It is easy to see the lower diagonal block matrix $\tilde{\Psi}$ is a Hurwitz matrix due to the fact that both diagonal blocks are Hurwitz.
Therefore, we have a symmetric matrix $U>0$ \cite{SriYin_19} such that 
\begin{align}
    \tilde{\Psi}^{T} U+U\tilde{\Psi} +I=0,  \nonumber 
\end{align}
which is referred to as the Lyapunov equation. 
For symmetric matrix $U$, there exist the largest and smallest eigenvalues $\lambda_{\max}$ and $\lambda_{\min}$, respectively. 
In addition, $\lambda_{\max}$ and $\lambda_{\min}$ are both positive. 
As a result, by \cite[Theorem~7]{SriYin_19}, the constants are: 
\begin{align}
    c_1&=\frac{0.9}{\lambda_{\max}}, \nonumber \\
    c_2&=2.25\frac{\lambda_{\max}}{\lambda_{\min}}, \nonumber \\
    c_3&=\frac{2\lambda^{2}_{\max}(r^{2}_{\max}+55(1+r_{\max})^{3})}{0.9\lambda_{\min}}. \nonumber
\end{align}
 
\section{Cooperative Navigation Task}
\label{sec: nav}
In this section, we provide further experimental details on cooperative navigation task in addition to Section \ref{sec: coop_navi}. Moreover, we use consensus error as another performance metric, which is defined as following
$$\text{CE}\left(\left\{w_k^i\right\}_{i=1}^{N}\right):=\frac{1}{N} \sum_{i \in \mathcal{N}}\left\|w_k^i-\overline{w}_k\right\|^2.$$
\begin{figure}[H]
    \centering
    \begin{subfigure}{.3\linewidth}
        \centering
        \includegraphics[width = \linewidth]{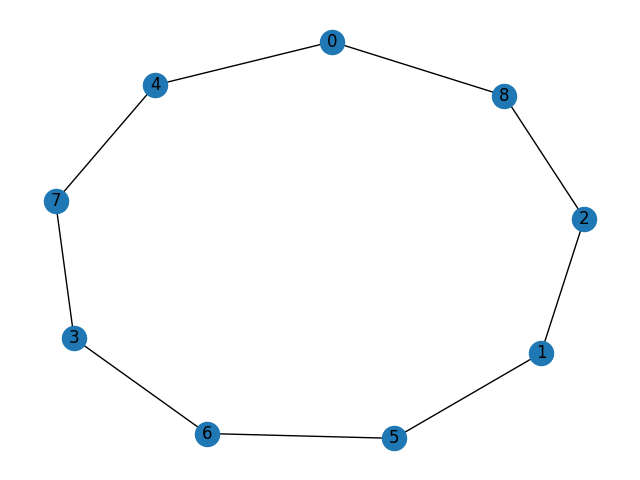}
        \caption{Ring \\Network}
    \end{subfigure}%
    \begin{subfigure}{.3\linewidth}
        \centering
        \includegraphics[width = \linewidth]{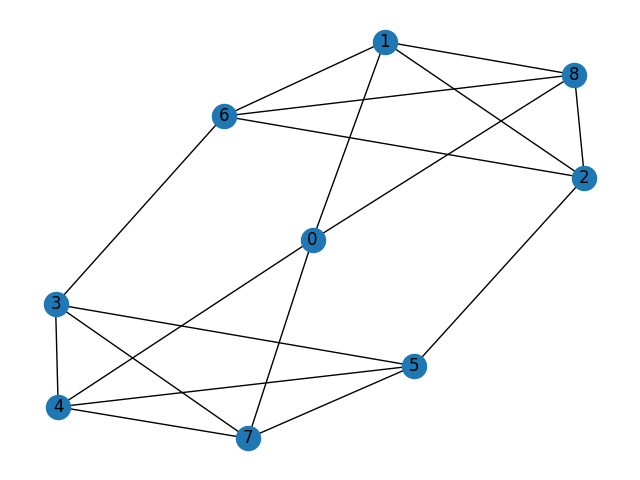}
        \caption{4-Regular \\Network}
    \end{subfigure}
    \begin{subfigure}{.3\linewidth}
        \centering
        \includegraphics[width = \linewidth]{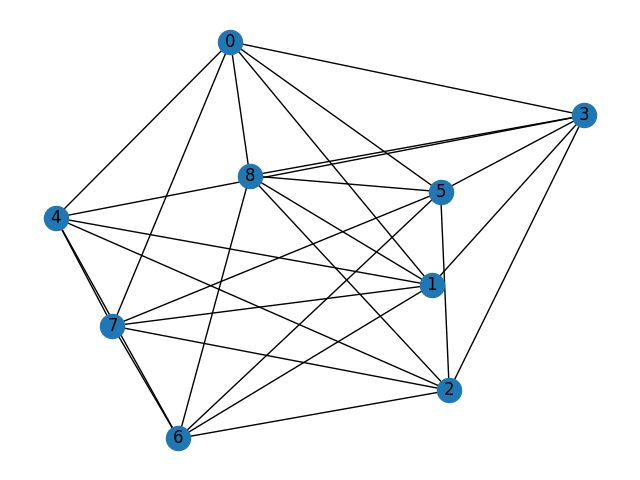}
        \caption{6-Regular \\Network}
    \end{subfigure}%
    
    \begin{subfigure}{.3\linewidth}
        \centering
        \includegraphics[width = \linewidth]{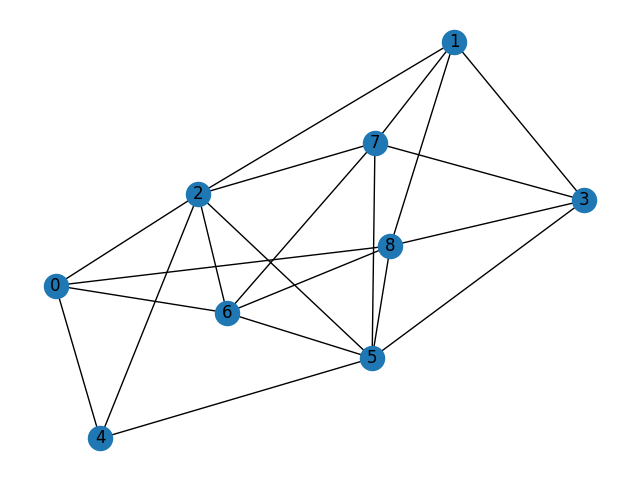}
        \caption{ER \\Network}
    \end{subfigure}
    \begin{subfigure}{.3\linewidth}
        \centering
        \includegraphics[width = \linewidth]{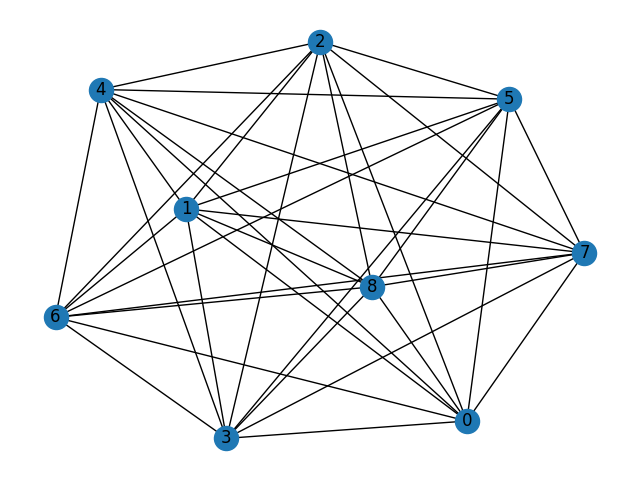}
        \caption{Complete \\Network}
    \end{subfigure}
    \caption{Network Topology}
    \label{fig: network}
    \vspace{-.2in}
\end{figure}

\subsection{Network Topology}
We compare all algorithms under five different network typologies. 
These are ring network, 4-regular network, 6-regular network, Erdos-Renyi(ER) network with 0.5 connection probability, and fully connected network. 
These network topologies are illustrated in Figure \ref{fig: network}. For simplicity, the local aggregation is the average of neighboring nodes for all networks.

\subsection{Convergence Performance}
First, we show the empirical convergence performances of all algorithms in terms of MSBEs and CEs.
Here we choose the number of local steps to be $K=20$ for local TD-update algorithm, the batch size to be $20$ for batch TD algorithm. 
For all algorithms, we set the total sample number to be $10000$. 
The comparisons among the algorithms are shown in Figure \ref{fig: all1}-\ref{fig: all4} over various network topologies. Left columns of the Figure \ref{fig: all1}-\ref{fig: all4} demonstrate the mean squared Bellman error(MSBE), and right columns demonstrate the consensus error(CE). 

We can see that all algorithms converge. More specifically, in terms of MSBE, the error floor of vanilla TD is the lowest, our proposed local TD-update approach is the second best and batch TD algorithm is the worst with a significant error gap. Similarly for CE, vanilla TD shows the lowest consensus error, our local TD-update approach shows slightly higher error, while batch TD algorithm shows the largest and oscillating consensus error across all network topologies.
This verifies our analysis that allowing local steps and performing infrequent communications is feasible and can converge.
In this parameter setting, vanilla TD algorithm performs 10000 communication rounds, which is the most, batch TD algorithm performs 1000 communication rounds, while local TD-update algorithm only performs 500 communication rounds. For more details, see discussion on the communication round in Section \ref{sec: navi_coop_comm}.

In Figure \ref{fig: topo_local}, we present the topology effect on our proposed algorithm. It is, in general, as the network becomes more and more connected the consensus error fluctuates less. Intuitively, with denser network, after local consensus aggregation, the parameter can be closer to the global average of the network.

\begin{figure}[H]
    \centering
    \begin{subfigure}{.47\linewidth}
        \centering
        \includegraphics[width = \linewidth]{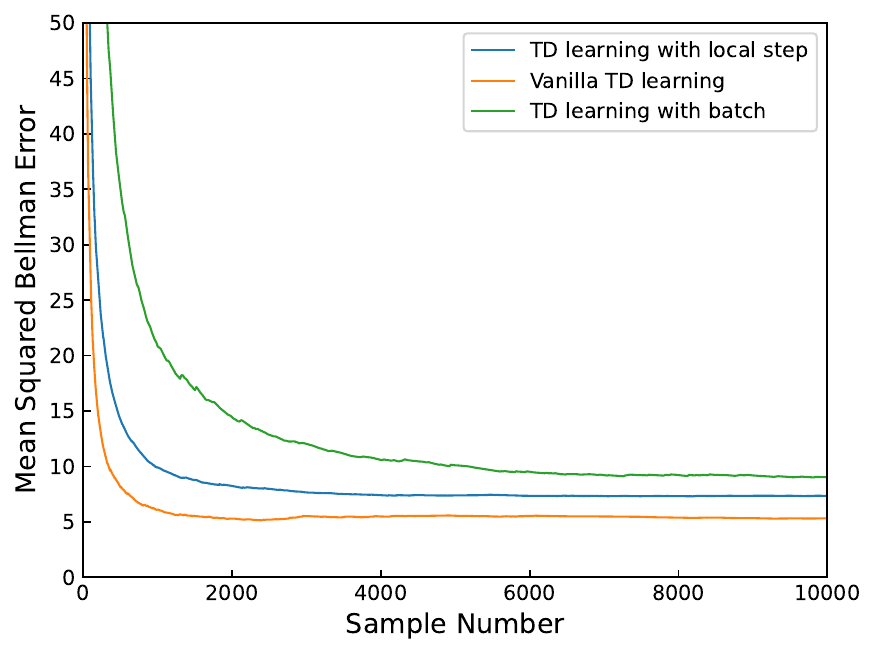}
        \caption{Mean Squared Bellman Error}
    \end{subfigure}%
    \begin{subfigure}{.47\linewidth}
        \centering
        \includegraphics[width = \linewidth]{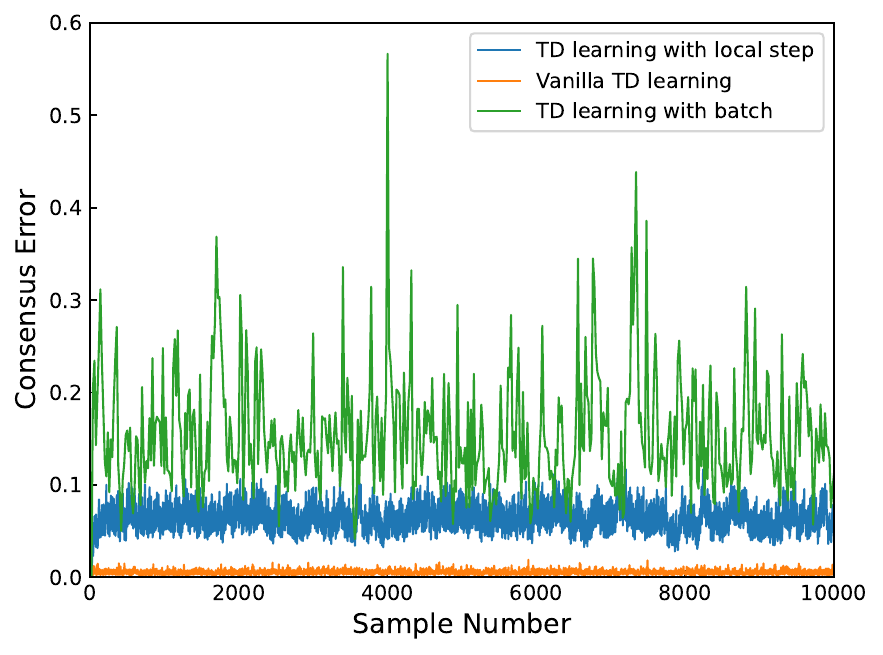}
        \caption{Consensus Error}
    \end{subfigure}%
    \caption{Comparison among Algorithms in Ring Network}
    \label{fig: all1}
    \vspace{-.2in}
\end{figure}

\begin{figure}[H]
    \centering
    \begin{subfigure}{.47\linewidth}
        \centering
        \includegraphics[width = \linewidth]{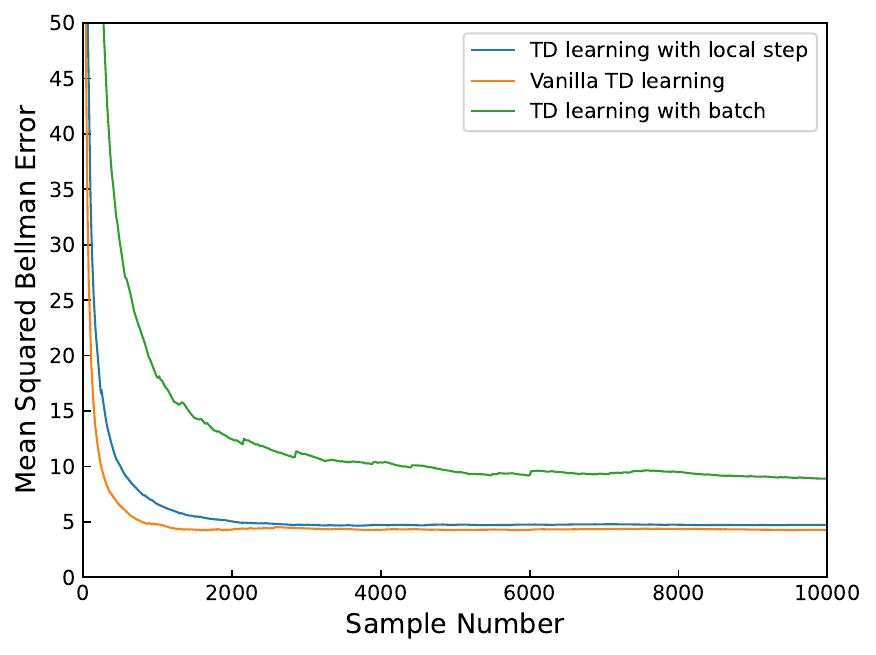}
        \caption{Mean Squared Bellman Error}
    \end{subfigure}%
    \begin{subfigure}{.47\linewidth}
        \centering
        \includegraphics[width = \linewidth]{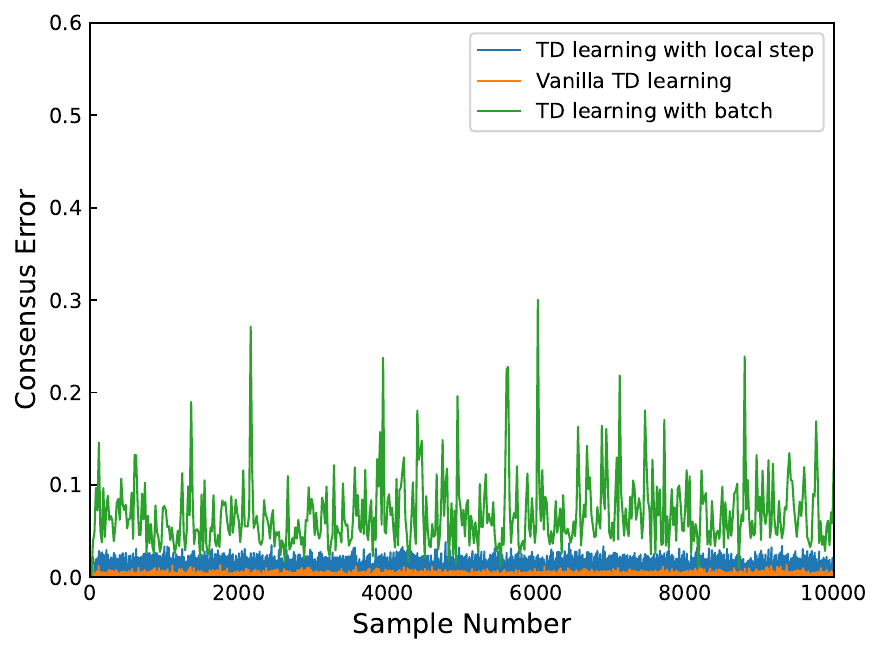}
        \caption{Consensus Error}
    \end{subfigure}%
    \caption{Comparison among Algorithms in 4-Regular Network}
    \label{fig: all2}
    \vspace{-.2in}
\end{figure}

\begin{figure}[H]
    \centering
    \begin{subfigure}{.47\linewidth}
        \centering
        \includegraphics[width = \linewidth]{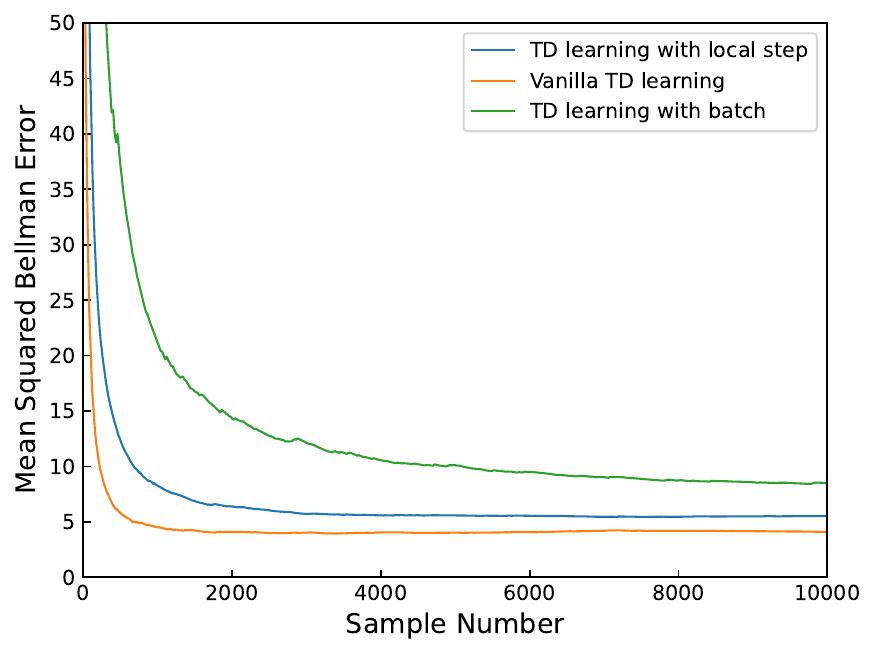}
        \caption{Mean Squared Bellman Error}
    \end{subfigure}%
    \begin{subfigure}{.47\linewidth}
        \centering
        \includegraphics[width = \linewidth]{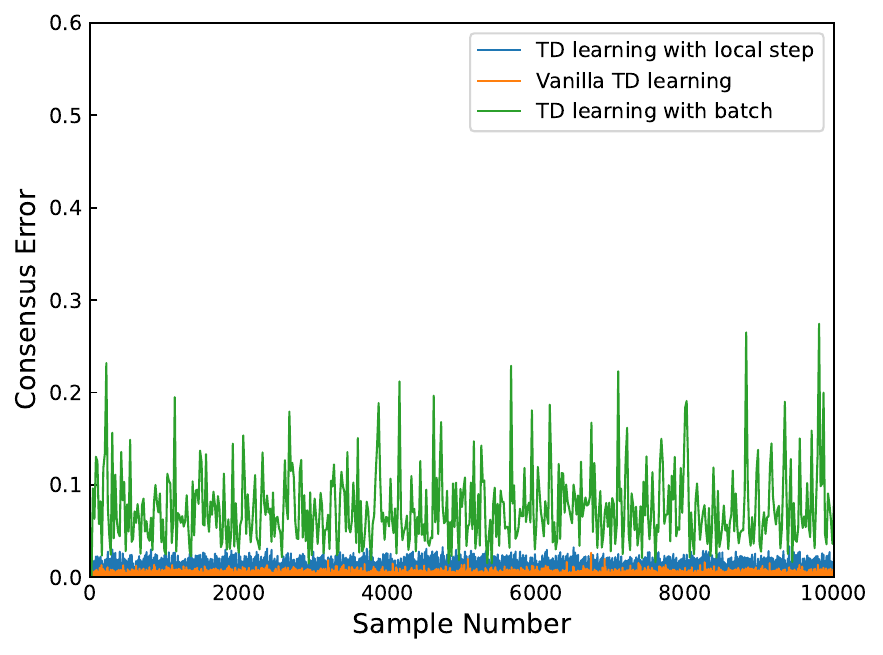}
        \caption{Consensus Error}
    \end{subfigure}%
    \caption{Comparison among Algorithms in 6-Regular Network}
    \label{fig: all3}
    \vspace{-.2in}
\end{figure}

\begin{figure}[H]
    \centering
    \begin{subfigure}{.47\linewidth}
        \centering
        \includegraphics[width = \linewidth]{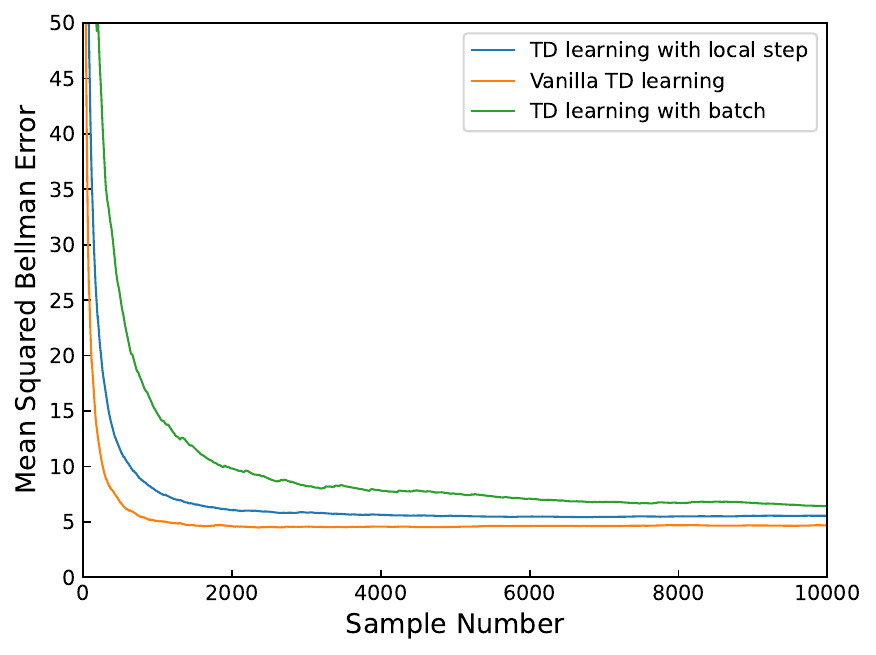}
        \caption{Mean Squared Bellman Error}
    \end{subfigure}%
    \begin{subfigure}{.47\linewidth}
        \centering
        \includegraphics[width = \linewidth]{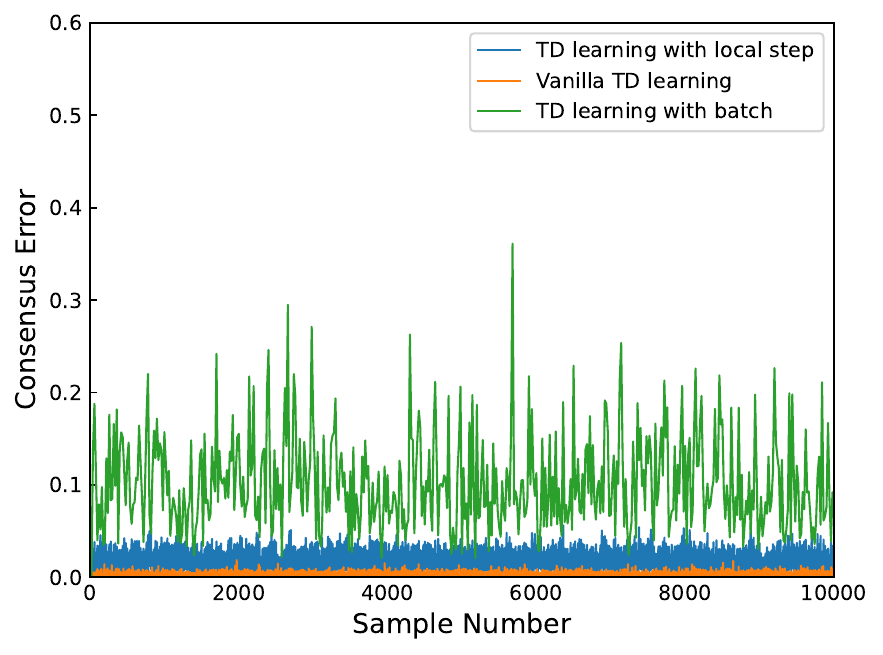}
        \caption{Consensus Error}
    \end{subfigure}%
    \caption{Comparison among Algorithms in ER Network}
    \label{fig: all5}
    \vspace{-.2in}
\end{figure}

\begin{figure}[H]
    \centering
    \begin{subfigure}{.47\linewidth}
        \centering
        \includegraphics[width = \linewidth]{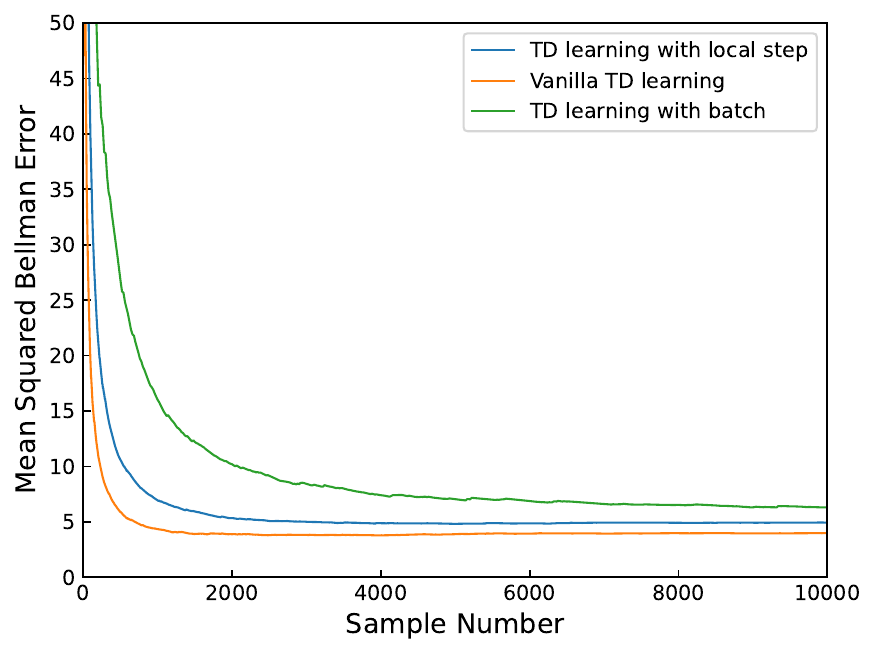}
        \caption{Mean Squared Bellman Error}
    \end{subfigure}%
    \begin{subfigure}{.47\linewidth}
        \centering
        \includegraphics[width = \linewidth]{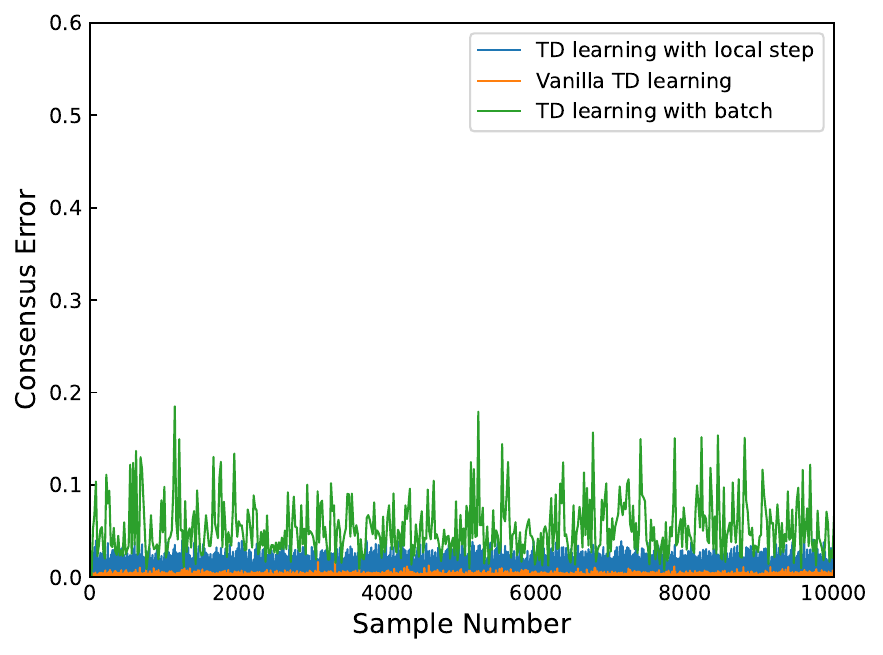}
        \caption{Consensus Error}
    \end{subfigure}%
    \caption{Comparison among Algorithms in Complete Network}
    \label{fig: all4}
    \vspace{-.2in}
\end{figure}

\begin{figure}[H]
  \centering
  \vspace{-.1in}
  \includegraphics[width=.45\textwidth]{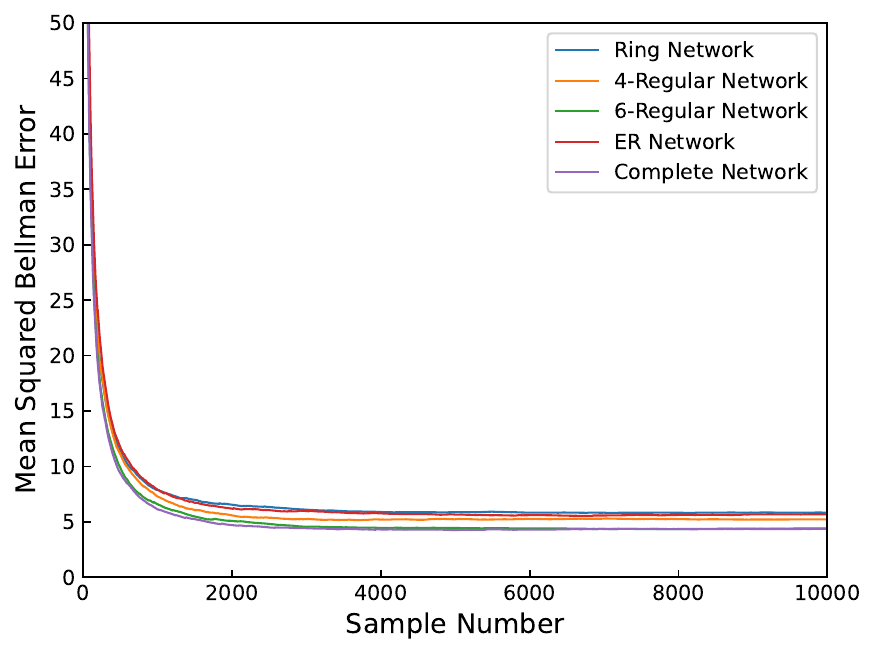}
  \caption{Topology on Local TD Algorithm}
  \label{fig: topo_local}
  \vspace{-.1in}
\end{figure}

In addition, we have compared different pairs of local step and communication round for our proposed algorithm and different pairs of batch size and communication round for batch TD algorithm in Figure \ref{fig: add_comp_task} for cooperative navigation task. In this setting, all algorithms converge to a similar Bellman error level. However, the convergence for local TD algorithm is faster than batch TD algorithms in all parameter settings. Moreover, with the increase of batch size, batch TD algorithm seems to converge slower while with the increase of local step, local TD algorithm convergence increases in general but not significantly.
\begin{figure}[H]
  \centering
  \vspace{-.1in}
  \includegraphics[width=.45\textwidth]{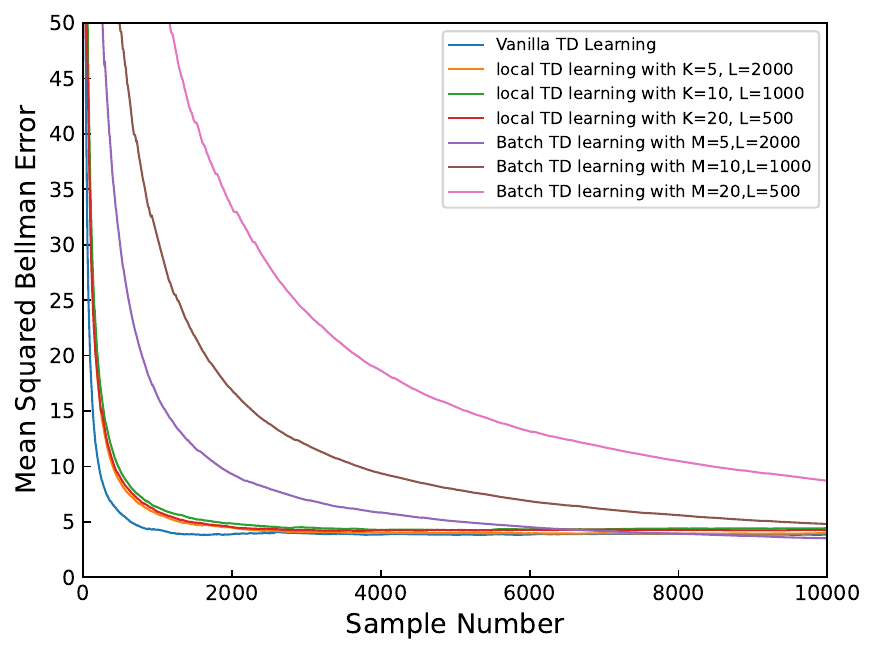}
  \caption{Additional Comparisons}
  \label{fig: add_comp_task}
  \vspace{-.1in}
\end{figure}

\subsection{Convergence Performance With Respect to Communication Rounds}
\label{sec: navi_coop_comm}
In Figure \ref{fig: comm}, we provide the convergence results with respect to the communication rounds for all algorithms, where the local step $K=20$ for local TD-update algorithm and the batch size is 20 for batch TD algorithm. We can see that within 500 communication rounds, local TD-update algorithm converges and requires much less communication round than vanilla TD algorithm, the convergence of which requires more than 1000 communication rounds. On the other hand, local TD-update approach converges to a lower error floor compared to batch TD algorithm. We can observe such empirical results across all network topologies.
\begin{figure*}
    \centering
    \begin{subfigure}{.2\linewidth}
        \centering
        \includegraphics[width = \linewidth]{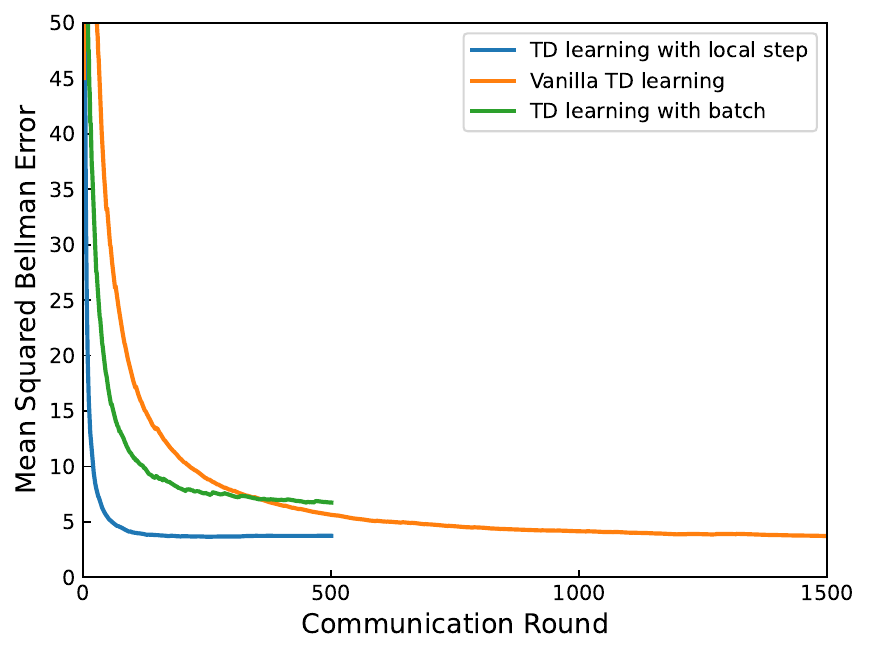}
        \caption{ring network}
    \end{subfigure}%
    \begin{subfigure}{.2\linewidth}
        \centering
        \includegraphics[width = \linewidth]{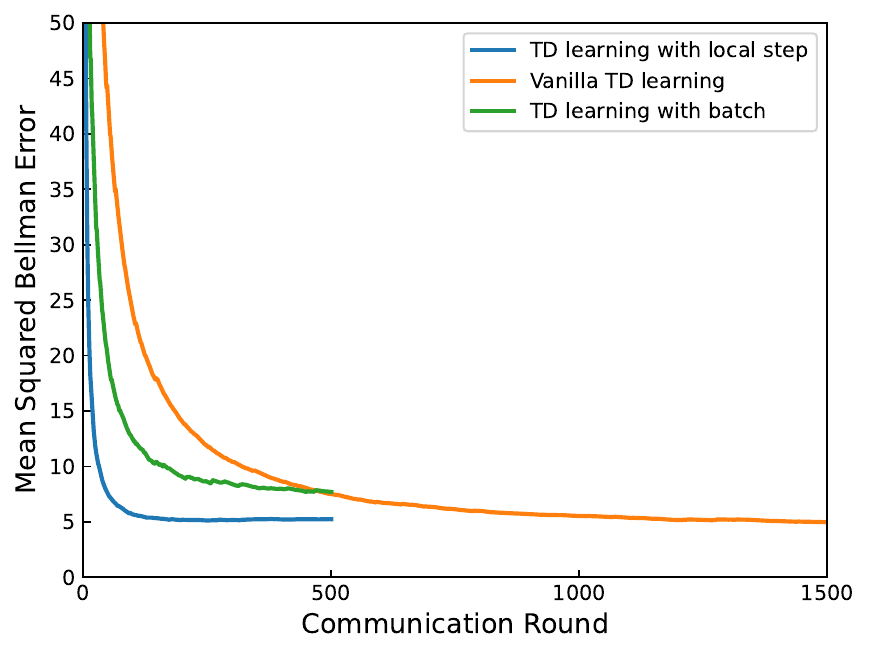}
        \caption{4-regular network}
    \end{subfigure}%
    \begin{subfigure}{.2\linewidth}
        \centering
        \includegraphics[width = \linewidth]{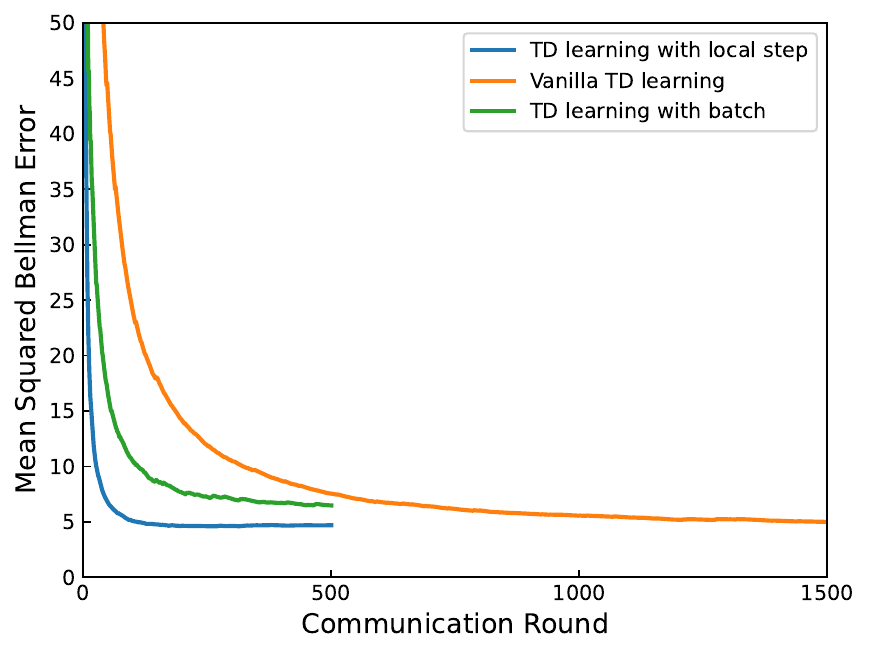}
        \caption{6-regular network}
    \end{subfigure}%
    \begin{subfigure}{.2\linewidth}
        \centering
        \includegraphics[width = \linewidth]{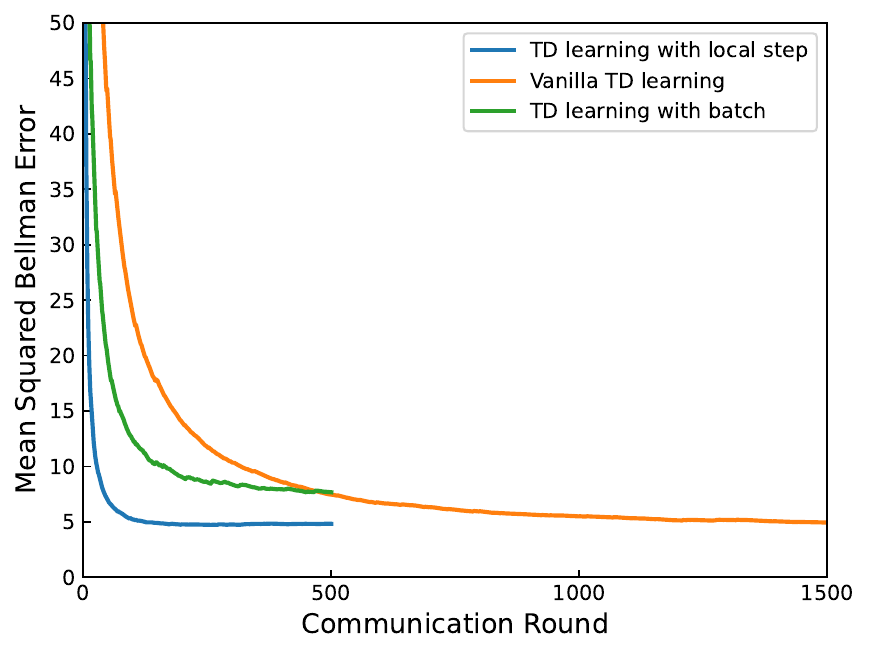}
        \caption{ER network}
    \end{subfigure}%
    \begin{subfigure}{.2\linewidth}
        \centering
        \includegraphics[width = \linewidth]{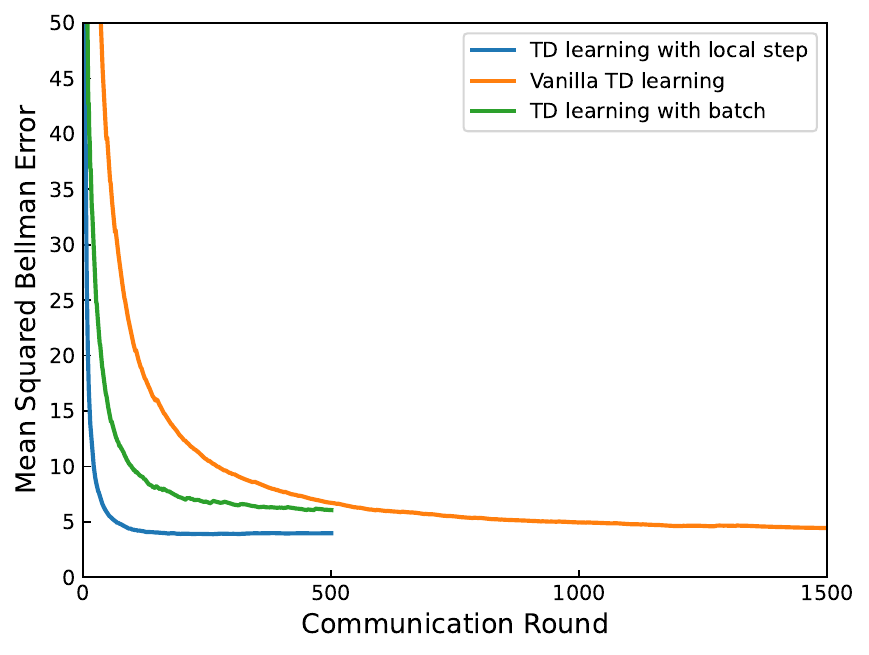}
        \caption{complete network}
    \end{subfigure}%
    \caption{Convergence Respect to Communication Rounds}
    \label{fig: comm}
    \vspace{-.2in}
\end{figure*}

\subsection{Impacts of the Number of Local Steps on Convergence}

Next, we illustrate the effect of the number of local steps $K$ on the convergence for our proposed algorithm. 
In Figure \ref{fig: local}, we vary the number of local steps from $K=10$ to $K=200$.
The right column is the consensus error of first 2000 samples, which displays a better view.
As the number of local steps increases, the mean squared Bellman error converges to a higher error level, on the other hand, the consensus error oscillates more.
In summary, larger local steps helps with saving more communication cost while it also result in converging to a higher mean squared Bellman error and a greater fluctuation of the consensus error. This conclusion echoes the results in synthetic experiment in Figure \ref{fig: local_steps}.

\begin{figure*}
    \centering
    \begin{subfigure}{.3\linewidth}
        \centering
        \includegraphics[width = \linewidth]{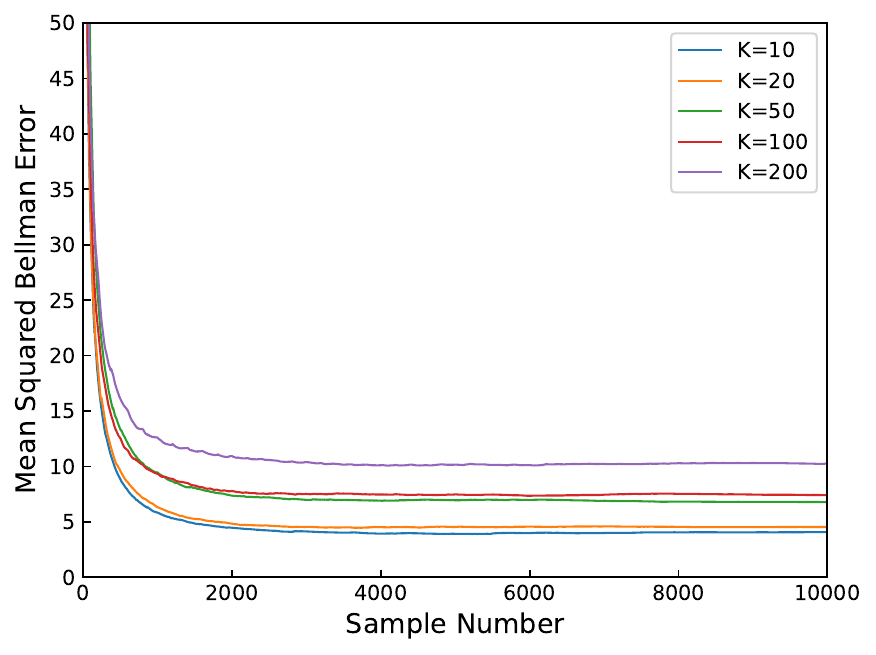}
        \caption{Mean Squared Bellman Error}
    \end{subfigure}%
    \begin{subfigure}{.3\linewidth}
        \centering
        \includegraphics[width = \linewidth]{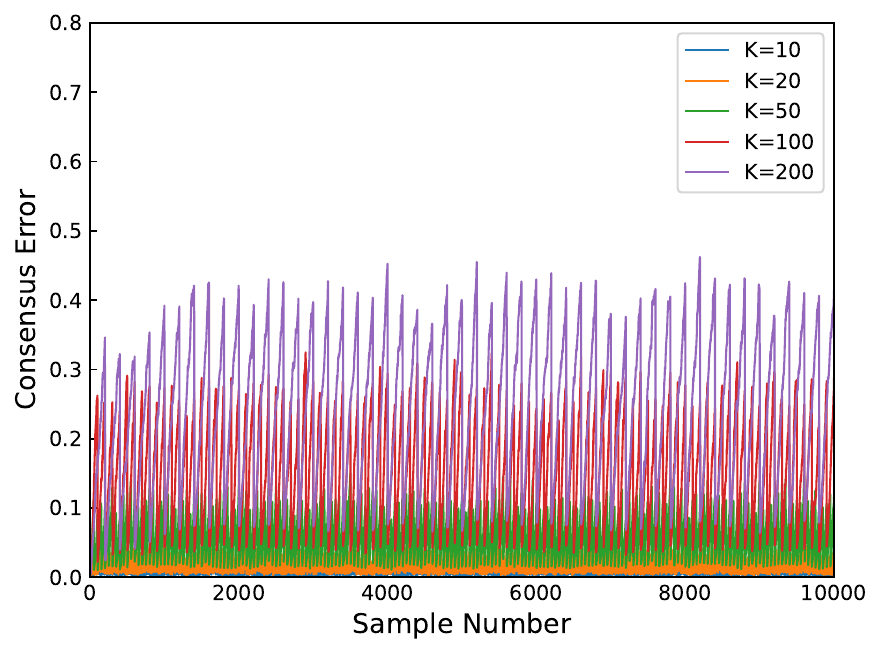}
        \caption{Consensus Error}
    \end{subfigure}%
    \begin{subfigure}{.3\linewidth}
        \centering
        \includegraphics[width = \linewidth]{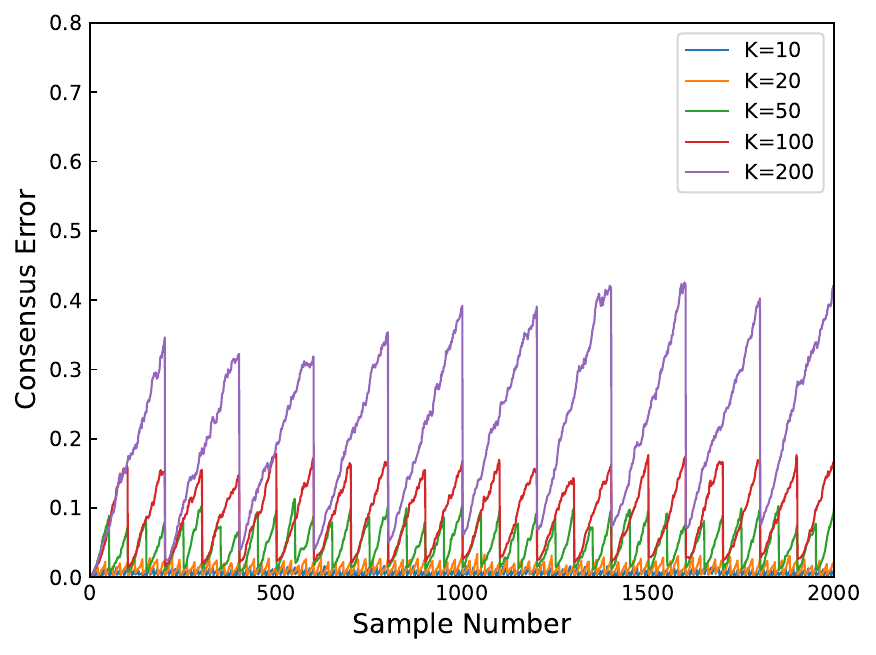}
        \caption{Zoomed Consensus Error}
    \end{subfigure}%
    \caption{Comparison among Different Local Steps in 4 Regular Network}
    \label{fig: local}
    \vspace{-.2in}
\end{figure*}

\subsection{Impacts of Step Size on Convergence}
Here, we illustrate the effect of step size $\beta$ on the convergence for our proposed algorithm and batch TD algorithm over 4-regular network to shed lights on the choice of step sizes for batch TD algorithm and our proposed local TD algorithm.
First of all, Figure \ref{fig: batch5}-\ref{fig: batch20} show the performance of batch TD algorithm over various batch sizes.
The purple line shows the performance of our proposed algorithm as a baseline comparison.
In general, in terms of mean squared Bellman error, larger step size $\beta$ leads to faster convergence speed and larger error level. 
Smaller step size $\beta$ leads to slower convergence speed, but could eventually converge to a smaller error floor.
Also, larger step size $\beta$ results in a greater consensus error.
Thus, we set step size $\beta=0.1$ and batch size to be $20$ for batch TD algorithm.

\begin{figure}[H]
    \centering
    \begin{subfigure}{.5\linewidth}
        \centering
        \includegraphics[width = \linewidth]{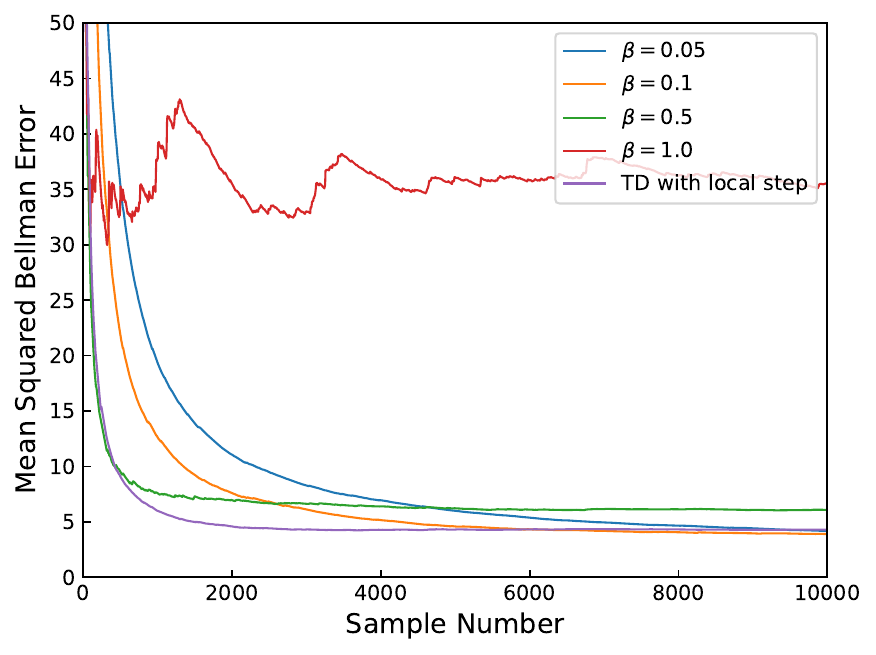}
        \caption{Mean Squared Bellman Error}
    \end{subfigure}%
    \begin{subfigure}{.5\linewidth}
        \centering
        \includegraphics[width = \linewidth]{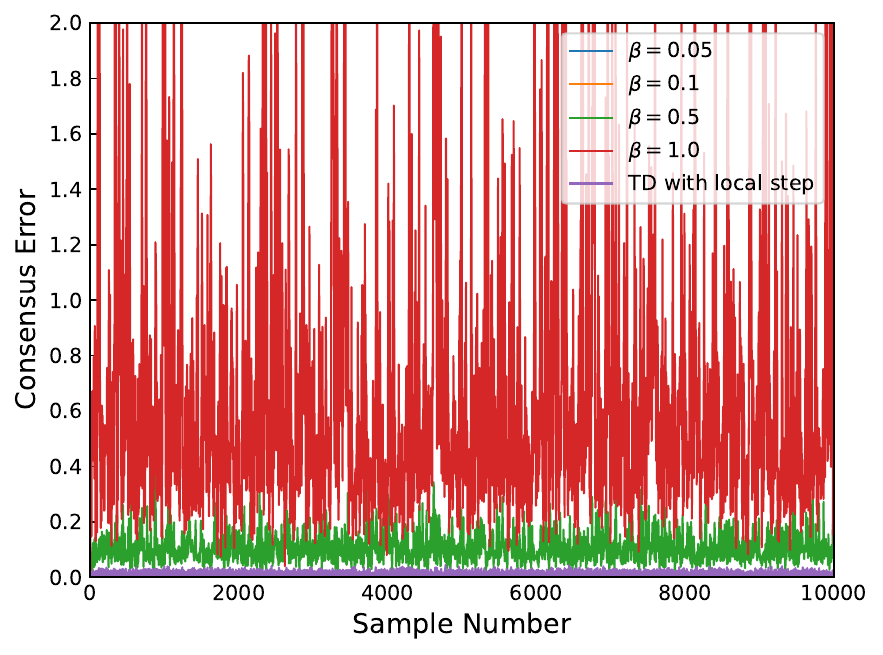}
        \caption{Consensus Error}
    \end{subfigure}%
    \caption{Comparison among Different step sizes $\beta$ for TD Learning with Batch Size 5 in 4 Regular Network}
    \label{fig: batch5}
    \vspace{-.2in}
\end{figure}

\begin{figure}[H]
    \centering
    \begin{subfigure}{.5\linewidth}
        \centering
        \includegraphics[width = \linewidth]{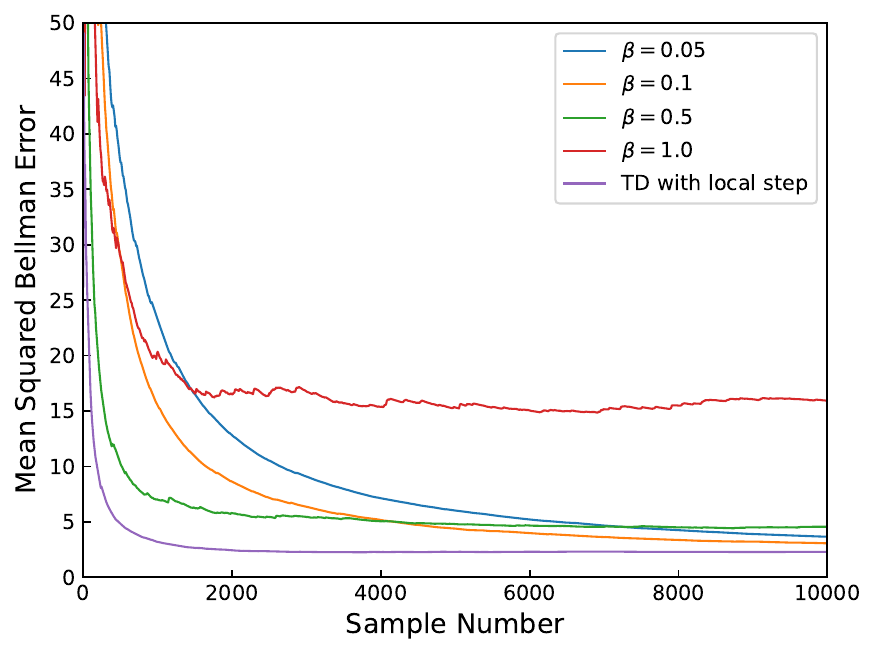}
        \caption{Mean Squared Bellman Error}
    \end{subfigure}%
    \begin{subfigure}{.5\linewidth}
        \centering
        \includegraphics[width = \linewidth]{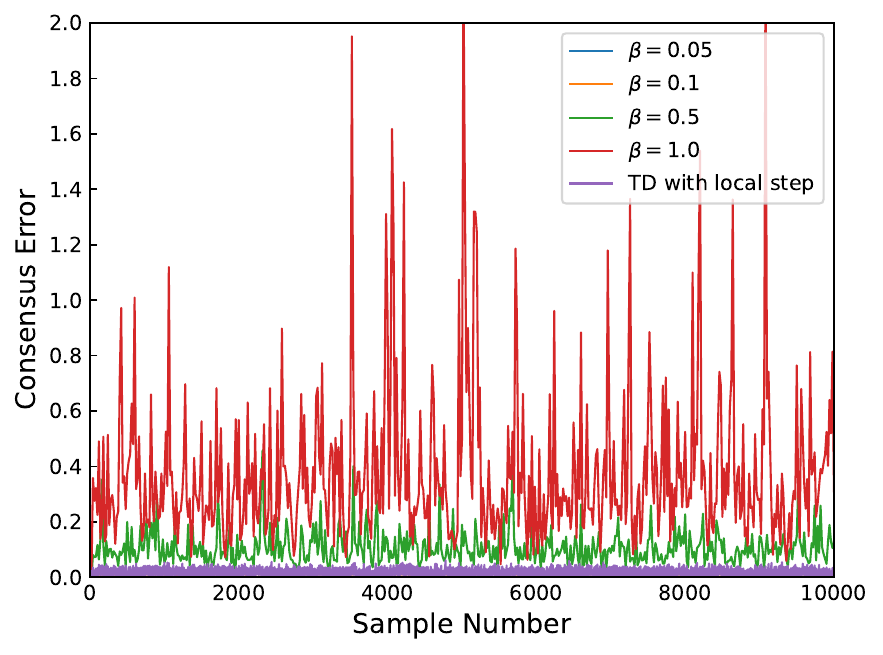}
        \caption{Consensus Error}
    \end{subfigure}%
    \caption{Comparison among Different step sizes $\beta$ for TD Learning with Batch Size 10 in 4 Regular Network}
    \label{fig: batch10}
    \vspace{-.2in}
\end{figure}

\begin{figure}[H]
    \centering
    \begin{subfigure}{.5\linewidth}
        \centering
        \includegraphics[width = \linewidth]{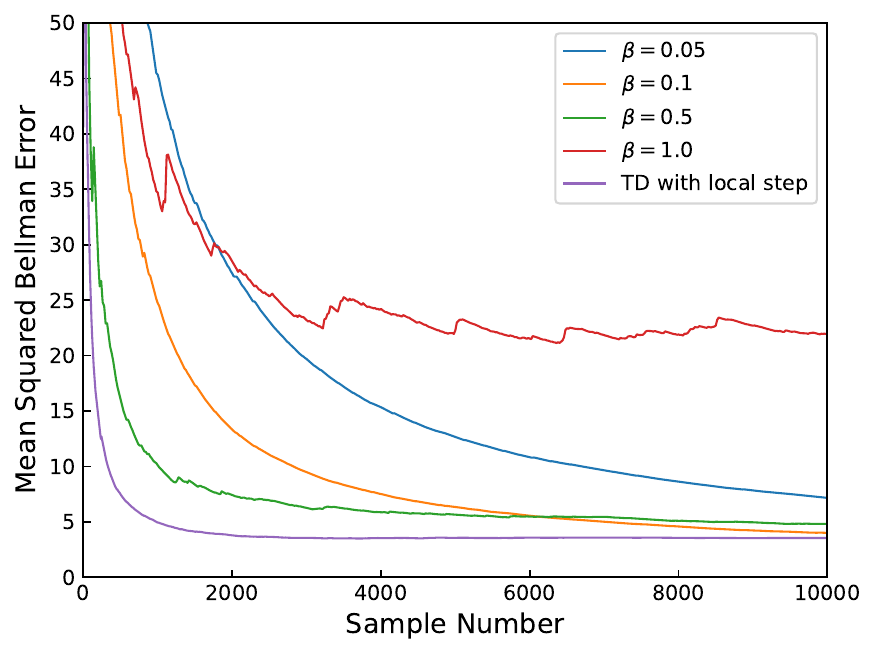}
        \caption{Mean Squared Bellman Error}
    \end{subfigure}%
    \begin{subfigure}{.5\linewidth}
        \centering
        \includegraphics[width = \linewidth]{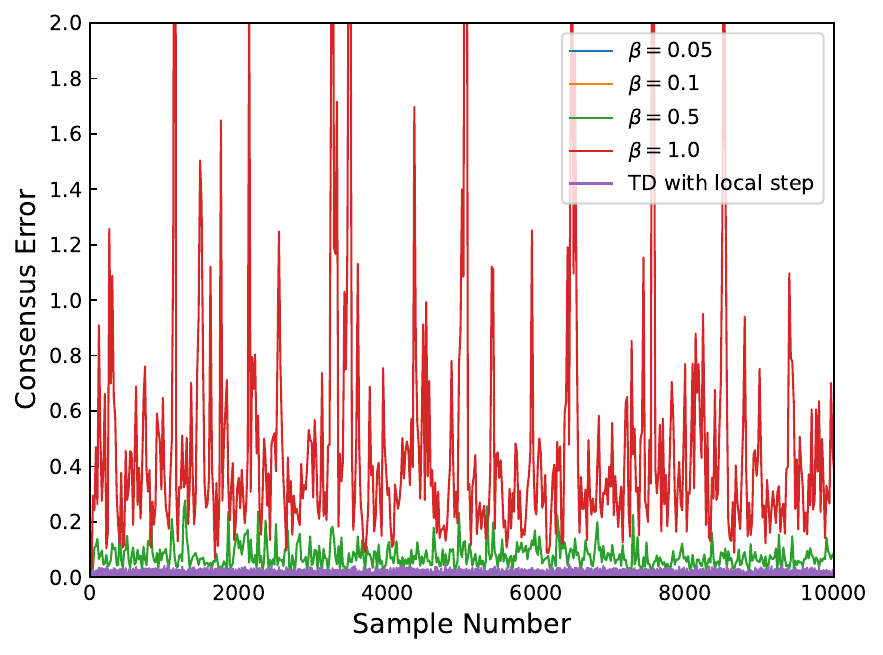}
        \caption{Consensus Error}
    \end{subfigure}%
    \caption{Comparison among Different step size $\beta$ for TD Learning with Batch Size 20 in 4 Regular Network}
    \label{fig: batch20}
    \vspace{-.2in}
\end{figure}
In Figure \ref{fig: local10}-\ref{fig: local50}, we show the effect of step size $\beta$ on the convergence performance of local TD-update algorithm.
Similar to batch TD algorithm, larger step size $\beta$ leads to faster convergence speed and larger mean squared Bellman and consensus error floor in local TD-update algorithm. 
However, the error floor differences among different step sizes $\beta$ are smaller compared to batch TD algorithm.
In order to balance among convergence speed, error floor, and communication cost, we decide to use step size $\beta=0.05$ and local step $K=20$.

\begin{figure}[H]
    \centering
    \begin{subfigure}{.5\linewidth}
        \centering
        \includegraphics[width = \linewidth]{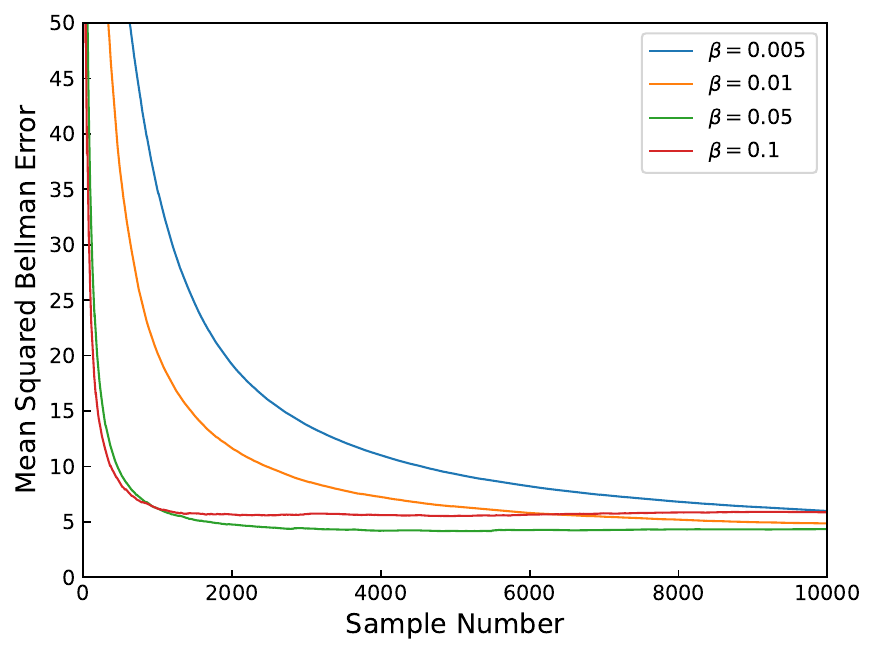}
        \caption{Mean Squared Bellman Error}
    \end{subfigure}%
    \begin{subfigure}{.5\linewidth}
        \centering
        \includegraphics[width = \linewidth]{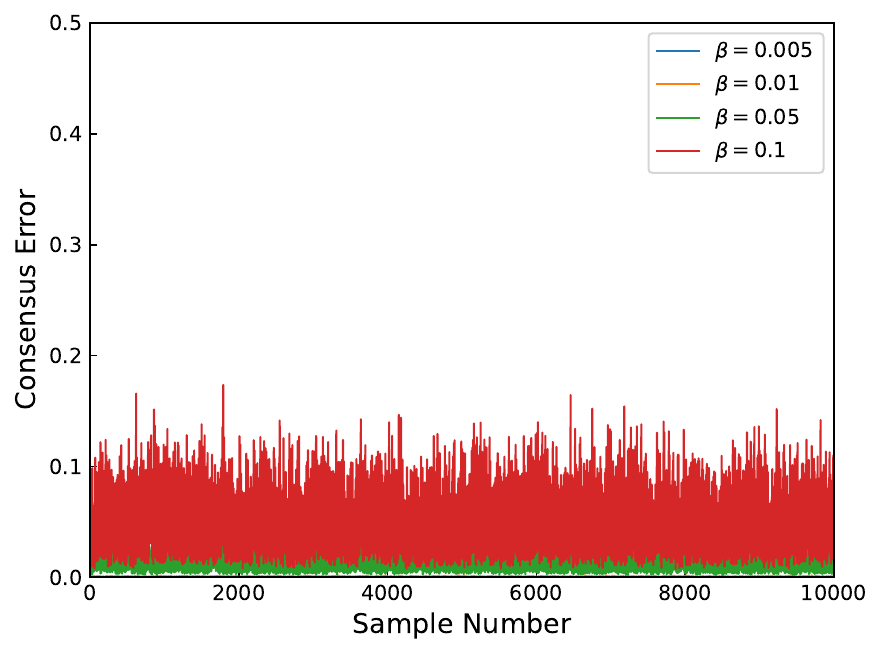}
        \caption{Consensus Error}
    \end{subfigure}%
    \caption{Comparison among Different step sizes $\beta$ for TD Learning with Local Step $K=10$ in 4 Regular Network}
    \label{fig: local10}
    \vspace{-.2in}
\end{figure}

\begin{figure}[H]
    \centering
    \begin{subfigure}{.5\linewidth}
        \centering
        \includegraphics[width = \linewidth]{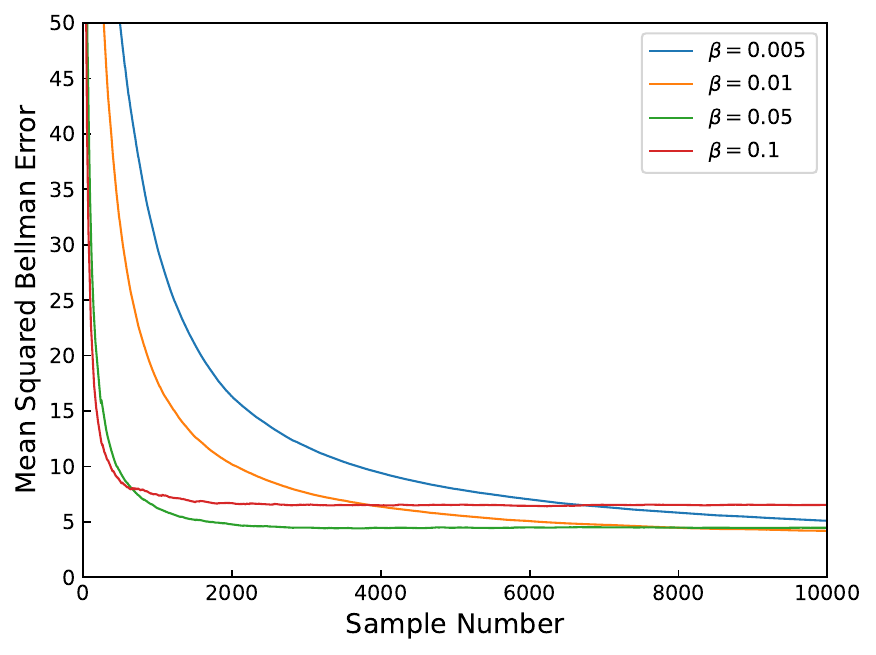}
        \caption{Mean Squared Bellman Error}
    \end{subfigure}%
    \begin{subfigure}{.5\linewidth}
        \centering
        \includegraphics[width = \linewidth]{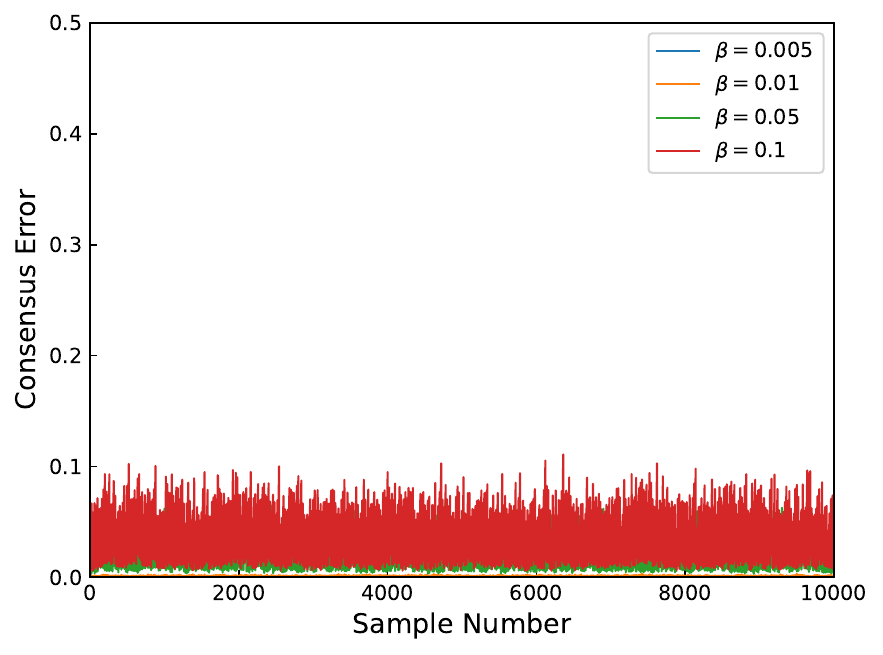}
        \caption{Consensus Error}
    \end{subfigure}%
    \caption{Comparison among Different step sizes $\beta$ for TD Learning with Local Step $K=20$ in 4 Regular Network}
    \label{fig: local20}
    \vspace{-.2in}
\end{figure}

\begin{figure}[H]
    \centering
    \begin{subfigure}{.5\linewidth}
        \centering
        \includegraphics[width = \linewidth]{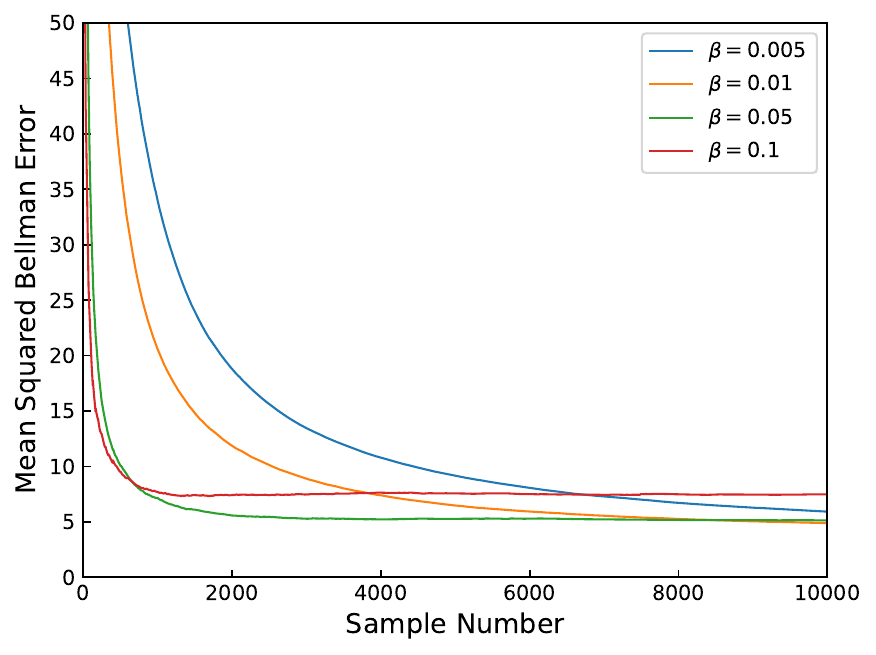}
        \caption{Mean Squared Bellman Error}
    \end{subfigure}%
    \begin{subfigure}{.5\linewidth}
        \centering
        \includegraphics[width = \linewidth]{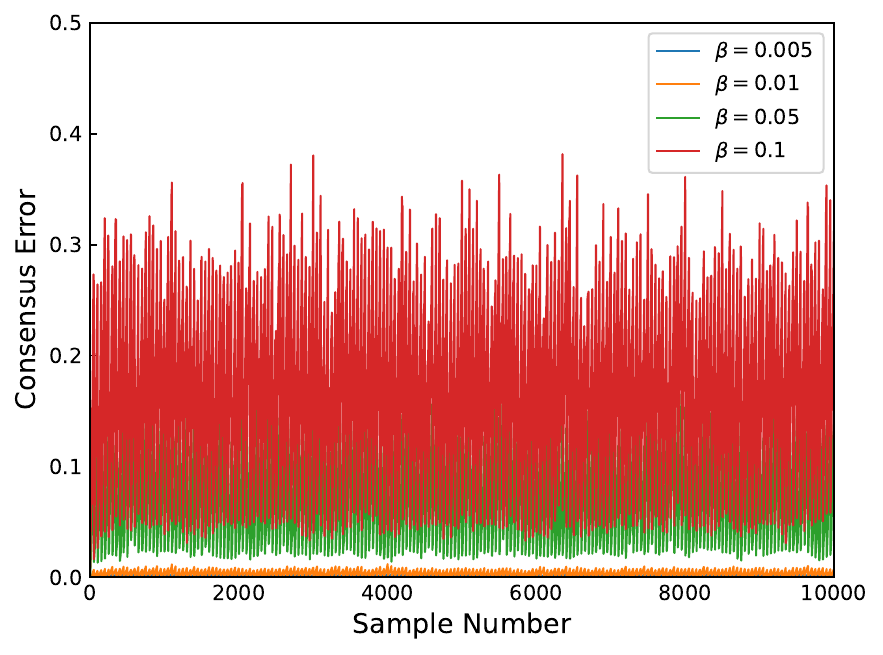}
        \caption{Consensus Error}
    \end{subfigure}%
    \caption{Comparison among Different step sizes $\beta$ for TD Learning with Local Step $K=50$ in 4 Regular Network}
    \label{fig: local50}
    \vspace{-.2in}
\end{figure}

\section{Proofs of Lemma and Theorem For Average Reward Setting}
In this section, we provide the derivation for the consensus error. Then, we prove the Lemma \ref{lem: con_err} and Theorem \ref{thm}.
\subsection{The derivation of consensus error}
Within each communication round $0\le l\le L-1$, the parameter update $w^{i}_{l,k}$ for agent $i$ at local step $k$ can be written as follows
\begin{align}
&w^{i}_{l,k+1} \nonumber \\
&=w^{i}_{l,k}+\beta \delta^{i}_{l,k}\cdot\phi(s_{l,k}) \nonumber \\
&=\left(I+\beta\phi(s_{l,k})[\phi(s_{l,k+1})-\phi(s_{l,k})]^{T}\right)w^{i}_{l,k}
+\beta (r^{i}_{l,k+1}-\mu^{i}_{l,k})\phi(s_{l,k}) \nonumber \\
&=B_{l,k}w^{i}_{l,k}+c^{i}_{l,k} 
\label{eq: w_i}
\end{align}
where $B_{l,k}:=I+\beta\phi(s_{l,k})[\phi(s_{l,k+1})-\phi(s_{l,k})]^{T}$ and $c^{i}_{l,k}:=\beta (r^{i}_{l,k+1}-\mu^{i}_{l,k})\phi(s_{l,k})$.
Then, from local step 0 to $K-1$, we have
\begin{align}
w^{i}_{l,K}&=\prod_{k=0}^{K-1}B_{l,k}w^{i}_{l,0}+\sum_{k=0}^{K-1}\prod_{n=k+1}^{K-1}B_{l,n} c^{i}_{l,k}. \nonumber
\end{align}

After a consensus update, the parameter for agent $i$ will be 
\begin{align}
w^{i}_{l+1,0}&=\sum_{j\in\mathcal{N}_i} A(i,j)\cdot w^{j}_{l,K} \nonumber \\
&=\sum_{j\in\mathcal{N}_i} A(i,j)\cdot \left(\prod_{k=0}^{K-1}B_{l,k}w^{j}_{l,0}+\sum_{k=0}^{K-1}\prod_{n=k+1}^{K-1}B_{l,n} c^{j}_{l,k}\right) \nonumber \\
&=\sum_{j\in\mathcal{N}_i} A(i,j)\cdot \prod_{k=0}^{K-1}B_{l,k}w^{j}_{l,0}+\sum_{j\in\mathcal{N}_i} A(i,j)\cdot \sum_{k=0}^{K-1}\prod_{n=k+1}^{K-1}B_{l,n} c^{j}_{l,k}.   \nonumber
\end{align}
The equation above shows the parameter update between two consecutive communication rounds.
Now we consider the average dynamics of the parameters across all agents. Recall $\bar{w}_{l,k}=\frac{1}{N}\sum_{i\in\mathcal{N}}w^{i}_{l,k}$, then within each communication round $0\le l\le L-1$, using \eqref{eq: w_i} we have 
\begin{align}
\bar{w}_{l,k+1}&=\frac{1}{N}\sum_{i\in\mathcal{N}}w^{i}_{l,k+1} \nonumber \\
&=\frac{1}{N}\sum_{i\in\mathcal{N}} (B_{l,k}w^{i}_{l,k}+c^{i}_{l,k}) \nonumber \\
&=B_{l,k}\bar{w}_{l,k}+\frac{1}{N}\sum_{i\in\mathcal{N}}c^{i}_{l,k} \nonumber \\
&=B_{l,k}\bar{w}_{l,k}+\bar{c}_{l,k}
\label{eq: w_bar}
\end{align}
where $\bar{c}_{l,k}:=\frac{1}{N}\sum_{i\in\mathcal{N}}c^{i}_{l,k}$.
Hence, the average dynamics from local step 0 to $K-1$ will be
\begin{align}
\bar{w}_{l,K}&=\prod_{k=0}^{K-1}B_{l,k}\bar{w}_{l,0}+\sum_{k=0}^{K-1}\prod_{n=k+1}^{K-1}B_{l,n} \bar{c}_{l,k}. \nonumber
\end{align}
After a consensus update, we have
\begin{align}
\bar{w}_{l+1,0}=\bar{w}_{l,K}. \label{eq: w_bar_con}
\end{align}
Note that the equation above means that consensus step will not change the average dynamics and average dynamic will only be updated during local steps.

For an agent $i\in\mathcal{N}$, we consider the consensus error at communication round $l$ and local step $k$, where $0\le k\le K-1$ and recall $Q^{i}_{l,k}=w^{i}_{l,k}-\bar{w}_{l,k}$. Then, we have
\begin{align}
Q^{i}_{l,k+1}=&w^{i}_{l,k+1}-\bar{w}_{l,k+1} \nonumber \\
=&B_{l,k}w^{i}_{l,k}+c^{i}_{l,k}-B_{l,k}\bar{w}_{l,k}-\bar{c}_{l,k} \nonumber \\
=&B_{l,k}(w^{i}_{l,k}-\bar{w}_{l,k})+c^{i}_{l,k}-\bar{c}_{l,k} \nonumber \\
=&B_{l,k}Q^{i}_{l,k}+c^{i}_{l,k}-\bar{c}_{l,k}.  \nonumber
\end{align}
Then, for the matrix form $Q_{l,k}=[Q^{1}_{l,k},\cdots,Q^{N}_{l,k}]\in R^{d\times N}$, we have 
\begin{align}
Q_{l,k+1}=&B_{l,k}Q_{l,k}+C_{l,k}(I-\frac{1}{N}\mathbf{1}\mathbf{1}^{T}) \nonumber
\end{align}
where $C_{l,k}:=[c^{1}_{l,k} \cdots c^{N}_{l,k}]$ and $\mathbf{1}$ denotes the all-1 column vector.
Then, for communication round $l$, we have 
\begin{align}
Q_{l,K}=&\prod_{k=0}^{K-1} B_{l,k}Q_{l,0}+\sum_{t=0}^{K-1}\prod_{\tilde{t}>t}^{K-1}B_{l,\tilde{t}}C_{l,k}(I-\frac{1}{N}\mathbf{1}\mathbf{1}^{T}) \nonumber
\end{align}
After a consensus update, we have 
\begin{align}
w^{i}_{l+1,0}-\bar{w}_{l,k}&=\sum_{j\in\mathcal{N}_i} A(i,j)w^{j}_{l,K}-\bar{w}_{l,K} \nonumber \\
&=\sum_{j\in\mathcal{N}_i} A(i,j)(w^{j}_{l,K}-\bar{w}_{l,K})=\sum_{j\in\mathcal{N}_i} A(i,j)Q^{j}_{l,K}. \nonumber
\end{align}
As a result, we have
\begin{align}
Q_{l+1,0}&=Q_{l,K}A^{T} \nonumber \\
&=\prod_{k=0}^{K-1} B_{l,k}Q_{l,0}A^{T}+\sum_{t=0}^{K-1}\prod_{\tilde{t}>t}^{K-1}B_{l,\tilde{t}}C_{l,t}(I-\frac{1}{N}\mathbf{1}\mathbf{1}^{T})A^{T}. \nonumber
\end{align}
After $L$ communication rounds, we have 
\begin{align}
Q_{L,0}&=\prod_{l=0}^{L-1}\prod_{k=0}^{K-1}B_{l,k}Q_{0,0} (A^{T})^{L} \nonumber \\
&\quad+\sum_{l=0}^{L-1}\prod_{j=1}^{L-1-l}\prod_{k=0}^{K-1}B_{l+j,k}\sum_{t=0}^{K-1}\prod_{\tilde{t}>t}^{K-1} B_{l,\tilde{t}} C_{l,t}(I-\frac{1}{N}\mathbf{1}\mathbf{1}^{T})(A^{T})^{L-l} \nonumber 
\end{align}

Note that for the second term when $l=L-1$, inside the summation, the summand becomes 
$\sum_{t=0}^{K-1}\prod_{\tilde{t}>t}^{K-1} B_{L-1,\tilde{t}} C_{L-1,t}(I-\frac{1}{N}\mathbf{1}\mathbf{1}^{T})A^{T}$. In other words, the matrix multiplier in front becomes an identity matrix.

\subsection{Proof of Lemma \ref{lem: con_err}}
\label{sec: proof_lemma1}
The norm of the consensus error is following
\begin{align}
&||Q_{L,0}|| \nonumber \\
=&||\prod_{l=0}^{L-1}\prod_{k=0}^{K-1}B_{l,k}Q_{0,0} (A^{T})^{L} \nonumber \\
&\quad+\sum_{l=0}^{L-1}\prod_{j=1}^{L-1-l}\prod_{k=0}^{K-1}B_{l+j,k}\sum_{t=0}^{K-1}\prod_{\tilde{t}>t}^{K-1} B_{l,\tilde{t}} C_{l,t}(I-\frac{1}{N}\mathbf{1}\mathbf{1}^{T})(A^{T})^{L-l}|| \nonumber \\
\le&||\prod_{l=0}^{L-1}\prod_{k=0}^{K-1}B_{l,k}Q_{0,0} (A^{T})^{L}|| \nonumber \\
&+||\sum_{l=0}^{L-1}\prod_{j=1}^{L-1-l}\prod_{k=0}^{K-1}B_{l+j,k}\sum_{t=0}^{K-1}\prod_{\tilde{t}>t}^{K-1} B_{l,\tilde{t}} C_{l,t}(I-\frac{1}{N}\mathbf{1}\mathbf{1}^{T})(A^{T})^{L-l}||. 
\label{eq: w_i_w_bar}
\end{align}

Before obtaining bounds on the terms of the consensus error in \eqref{eq: w_i_w_bar}, we first provide some useful bounds on $B_{l,k}$ and $C_{l,k}$. First, we have
\begin{align}
||B_{l,k}||&=||I+\beta\phi(s_{l,k})[\phi(s_{l,k+1})-\phi(s_{l,k})]^{T}|| \nonumber \\
&\le 1+\beta||\phi(s_{l,k})||(||\phi(s_{l,k+1})||+||\phi(s_{l,k})|| ) \nonumber \\
&\le 1+2\beta \nonumber
\end{align}
where the second inequality is due to Assumption \ref{ass: fea}.
Then, we have $||C_{k,l}||\le 2\beta \sqrt{N} r_{\max}$, where $r_{\max}=\sup_{i,s,a} r^{i}(s,a)$ by Assumption \ref{ass: r_bou}. This is because 
\begin{align}
||C_{k,l}||&=||\beta\phi(s_{k,l})\left((r^{1}_{k,l+1}, \cdots, r^{N}_{l,k+1})-(\mu^{1}_{k,l}, \cdots, \mu^{N}_{l,k})\right)|| \nonumber \\
&\le \beta||\phi(s_{k,l})||\cdot (||(r^{1}_{k,l+1}, \cdots, r^{N}_{l,k+1})||+||(\mu^{1}_{k,l}, \cdots, \mu^{N}_{l,k})||) \nonumber \\
&=2\beta\sqrt{N}r_{\max}. \nonumber
\end{align}

Next, inspired by \cite{SriYin_19}, we want to use the following bound
\begin{align}
(1+x)^{K}\le 1+2xK \nonumber
\end{align}
for small $x$. Note that 
\begin{align}
(1+x)^{K}|_{x=0}=1+2xK|_{x=0} \nonumber
\end{align}
and when $x\le\frac{\log 2}{K-1}$,
\begin{align}
\frac{\partial}{\partial x}(1+x)^{K}=K(1+x)^{K-1}\le Ke^{x(K-1)}\le 2K=\frac{\partial}{\partial x}(1+2xK) \nonumber
\end{align}
where the first inequality is due to the fact $\log (1+x)\le x$ for $x\ge 0$ and the second inequality is due to the fact $x\le\frac{\log 2}{K-1}$. Let $2\beta=x$ and $\beta\le\frac{1}{2K}\le\frac{\log 2}{2(K-1)}$.

For the first term in \eqref{eq: w_i_w_bar}, when $\beta \le\frac{1}{2K}$, we have that 
\begin{align}
||\prod_{l=0}^{L-1}\prod_{k=0}^{K-1}B_{l,k}Q_{0,0} (A^{T})^{L}||&\le ||\prod_{l=0}^{L-1}\prod_{k=0}^{K-1}B_{l,k}||\cdot ||Q_{0,0} (A^{T})^{L}|| \nonumber \\
&\le \kappa(1+2\beta)^{KL}(1-\eta^{N-1})^{L} \nonumber \\
&\le \kappa(1+4\beta K)^{L}(1-\eta^{N-1})^{L} \nonumber \\
&= \kappa \rho^{L} \nonumber
\end{align}
where we define $\rho:=(1+4\beta K)(1-\eta^{N-1})$. When $0<\beta K<\min\{\frac{1}{2},\frac{\eta^{N-1}}{4(1-\eta^{N-1})}\}$, we have $0<\rho<1$. The second inequality comes from the following two results.

First, consider the case where $A$ is a symmetric matrix for simplicity, then we have
\begin{align}
||Q_{0,0}A^{L}_{1,:}||&=||Q_{0,0}A^{L}_{1,:}-Q_{0,0}\frac{1}{N}\mathbf{1}|| \nonumber \\
&=||\sum_{i\in\mathcal{N}}(A^{L}_{1,i}-\frac{1}{N})Q^{i}_{0,0}|| \nonumber \\
&\le\sum_{i\in\mathcal{N}}|A^{L}_{1,i}-\frac{1}{N}|\cdot ||Q^{i}_{0,0}|| \nonumber \\
&\le N\cdot 2\frac{1+\eta^{-(N-1)}}{1-\eta^{N-1}}(1-\eta^{N-1})^{L}\cdot\max_{i\in\mathcal{N}}||Q^{i}_{0,0}|| \nonumber \\
&\le 2N\frac{1+\eta^{-(N-1)}}{1-\eta^{N-1}}(1-\eta^{N-1})^{L}\cdot ||Q_{0,0}||, \nonumber
\end{align}
where the second inequality is from \cite{NedOzd_09} (Proposition 1). Hence, $||Q_{0,0}A^{L}||\le2N^{2}\frac{1+\eta^{-(N-1)}}{1-\eta^{N-1}}(1-\eta^{N-1})^{L}\cdot ||Q_{0,0}||=\kappa_1 (1-\eta^{N-1})^{L} ||Q_{0,0}||$, where $\kappa_1=2N^{2}\frac{1+\eta^{-(N-1)}}{1-\eta^{N-1}}$. Second, we have
\begin{align}
||\prod_{l=0}^{L-1}\prod_{k=0}^{K-1}B_{l,k}||&\le \prod_{l=0}^{L-1}\prod_{k=0}^{K-1}|| B_{l,k}|| \nonumber \\
&\le  \prod_{l=0}^{L-1}\prod_{k=0}^{K-1}(1+2\beta) =(1+2\beta)^{KL}. \nonumber
\end{align}

To bound the second term of \eqref{eq: w_i_w_bar}, we have
\begin{align}
|| (I-\frac{1}{N}\mathbf{1}\mathbf{1}^{T})A^{L-l}||&=||A^{L-l}-\frac{1}{N}\mathbf{1}\mathbf{1}^{T}|| \nonumber \\
&\le  2N^{2}(1+\eta^{-(N-1)})(1-\eta^{N-1})^{L-l-1}. \nonumber
\end{align}
where the inequality is also from \cite{NedOzd_09} (Proposition 1).
Then, we also have
\begin{align}
&\sum_{t=0}^{K-1}\prod_{\tilde{t}>t}^{K-1}||B_{l,\tilde{t}}||\cdot ||C_{l,t}|| \nonumber \\
\le&\sum_{t=0}^{K-1}(1+2\beta)^{K-1-t}\cdot 2\beta\sqrt{N} r_{\max} \nonumber \\
=&2\beta\sqrt{N} r_{\max} \sum_{t=0}^{K-1}(1+2\beta)^{K-1-t} \nonumber \\
\le&4\beta K \sqrt{N} r_{\max}. \nonumber 
\end{align}
Then, for the multipliers, we have
\begin{align}
||\prod_{j=1}^{L-1-l}\prod_{k=0}^{K-1}B_{l+j,k}||\le(1+2\beta)^{(L-l-1)K}\le (1+4\beta K)^{L-l-1}. \nonumber
\end{align}
Finally, for the second term in consensus error \eqref{eq: w_i_w_bar}, we have 
\begin{align}
&||\sum_{l=0}^{L-1}\prod_{j=1}^{L-1-l}\prod_{k=0}^{K-1}B_{l+j,k}\sum_{t=0}^{K-1}\prod_{\tilde{t}>t}^{K-1} B_{l,\tilde{t}} C_{l,t}(I-\frac{1}{N}\mathbf{1}\mathbf{1}^{T})(A^{T})^{L-l}||\nonumber \\
\le&\sum_{l=0}^{L-1}||\prod_{j=1}^{L-1-l}\prod_{k=0}^{K-1}B_{l+j,k}||\cdot||\sum_{t=0}^{K-1}\prod_{\tilde{t}>t}^{K-1} B_{l,\tilde{t}} C_{l,t}||\cdot||(I-\frac{1}{N}\mathbf{1}\mathbf{1}^{T})(A^{T})^{L-l}|| \nonumber \\
\le&\sum_{l=0}^{L-1}  (1+4\beta K)^{L-l-1} \cdot 4\beta K\sqrt{N}r_{\max} \cdot  2N^{2}(1+\eta^{-(N-1)})(1-\eta^{N-1})^{L-l-1} \nonumber \\
\le&\kappa_2 \beta K\sum_{l=0}^{L-1} \rho^{L-l-1} \nonumber \\
\le& \frac{\kappa_2 \beta K}{1-\rho} \nonumber
\end{align}
where $\kappa_2=8(1+\eta^{-(N-1)})N^{\frac{5}{2}}r_{\max}$.

As a result, we have the results consensus bound of \eqref{eq: con_err} in Lemma \ref{lem: con_err}.

\subsection{Proof of the Theorem \ref{thm}}
For the mean square error, we have
\begin{align}
&\mathbb{E}[\sum_{i=1}^{N}\|w^{i}_{L,0}-w^{*}\|^{2}] \nonumber \\
=&\mathbb{E}[\sum_{i=1}^{N}\|w^{i}_{L,0}-\bar{w}_{L,0}+\bar{w}_{L,0}-w^{*}\|^{2}] \nonumber \\
\le& 2\mathbb{E}[\sum_{i=1}^{N}\|w^{i}_{L,0}-\bar{w}_{L,0}\|^{2}]+2\mathbb{E}[\sum_{i=1}^{N}\|\bar{w}_{L,0}-w^{*}\|^{2}] \nonumber \\
\le& 2d\mathbb{E}[\|Q_{L,0}\|^{2}]+2N\mathbb{E}[\|\bar{w}_{L,0}-w^{*}\|^{2}]
\label{eq: last}
\end{align}
where the first inequality is due to $\|x+y\|^{2}\le 2\|x\|^{2}+2\|y\|^{2}$ and the second inequality $\|X\|_{F}\le \sqrt{d}\|X\|$ for $X\in \mathbb{R}^{d\times N}$. 
Then, the stated result in \eqref{eq: convergence} follows from Lemmas~\ref{lem: con_err} and \ref{lem: ave_con}, and \eqref{eq: last}.  

This concludes the proof.

\balance

\bibliographystyle{ACM-Reference-Format}
\bibliography{refs}



\end{document}